%% file: main.tex
\documentclass{article}

\usepackage{microtype}
\usepackage{graphicx}
\usepackage{subfigure}
\usepackage{booktabs} 
\usepackage{algorithm}
\usepackage{algorithmicx}
\usepackage[noend]{algpseudocode}
\usepackage{nicefrac}

\usepackage{hyperref}

\usepackage[accepted]{icml2018}

\icmltitlerunning{Beyond $0.5$-Approximation for Submodular Maximization on Massive Data Streams}

\usepackage{thmtools,thm-restate}
\usepackage[utf8]{inputenc}
\usepackage{enumerate}
\usepackage{amssymb}
\usepackage{amsmath}
\usepackage{graphicx}
\usepackage{ae}
\usepackage{xspace}
\usepackage[polish,english]{babel}
\usepackage{bbm}
\usepackage[amsmath,amsthm,thmmarks,hyperref]{ntheorem}
\usepackage[capitalise,noabbrev]{cleveref}
\usepackage{caption}

\newtheorem{theorem}{Theorem}[section]

\newtheorem{lemma}[theorem]{Lemma}
\newtheorem{observation}[theorem]{Observation}

\newtheorem{definition}[theorem]{Definition}
\newtheorem{claim}{Claim} 
\newtheorem{fact}[theorem]{Fact}

\qedsymbol{\openbox}

\DeclareMathOperator{\poly}{poly}

\newcommand{\cO}{\mathcal{O}}
\newcommand{\opt}{\cO}
\newcommand{\fopt}{\mathrm{OPT}\xspace}

\newcommand{\cT}{\mathcal{T}}

\newcommand{\eps}{\varepsilon}
\newcommand{\cA}{\mathcal{A}}

\newcommand{\cG}{\mathcal{G}}
\newcommand{\cE}{\mathcal{E}}

\newcommand{\cM}{\mathcal{M}}

\newcommand{\bR}{\mathbb{R}}

\newcommand{\bRplus}{\mathbb{R}_+}
\newcommand{\one}{\mathbbm{1}}

\newcommand{\ceil}[1]{\lceil #1 \rceil}

\newcommand{\ab}[1]{\left<{#1}\right>} 
\newcommand{\rb}[1]{\left( #1 \right)} 
\renewcommand{\sb}[1]{\left[ #1 \right]} 
\newcommand{\abs}[1]{\left| #1 \right|} 
\renewcommand{\Pr}[1]{\mathbb{P}\left[{#1}\right]}
\newcommand{\E}[1]{\mathbb{E}\left[{#1}\right]}
\newcommand{\marginal}[2]{f\rb{#1|#2}}

\newcommand{\SEXYP}{\textsc{SubMax}\xspace}
\newcommand{\greedy}{\textsc{Greedy}\xspace}
\newcommand{\sieve}{\textsc{Sieve-Streaming}\xspace}
\newcommand{\Sieve}{\sieve}
\newcommand{\ouralgo}{\textsc{Salsa}\xspace}
\newcommand{\twopass}{\textsc{Two-Pass}\xspace}
\newcommand{\ppass}{\textsc{P-Pass}\xspace}
\newcommand{\dense}{\textsc{Dense}\xspace}
\newcommand{\fixed}{\textsc{Fixed Threshold}\xspace}
\newcommand{\highlow}{\textsc{High-Low Threshold}\xspace}

\begin{document}
 
\twocolumn[
\icmltitle{Beyond $1/2$-Approximation for Submodular Maximization \\on Massive Data Streams}

\icmlsetsymbol{equal}{*}

\begin{icmlauthorlist}
\icmlauthor{Ashkan Norouzi-Fard}{equal,to}
\icmlauthor{Jakub Tarnawski}{equal,to}
\icmlauthor{Slobodan Mitrovi\'c}{equal,to}
\icmlauthor{Amir Zandieh}{equal,to}
\icmlauthor{Aida Mousavifar}{to}
\icmlauthor{Ola Svensson}{to}
\end{icmlauthorlist}

\icmlaffiliation{to}{Theory of Computation Laboratory, EPFL, Lausanne, Vaud, Switzerland}

\icmlcorrespondingauthor{Ashkan Norouzi-Fard}{ashkan.afn@gmail.com}

\icmlkeywords{Submodular Maximization, Streaming, Optimization, Monotone, ICML}

\vskip 0.3in
]



\printAffiliationsAndNotice{\icmlEqualContribution} 

\input{050-abstract}
\input{100-introduction}

\input{200-preliminaries}
\input{300-algorithm}
\input{500-experiments}

\input{600-conclusion}
\input{acknowledgements}

\bibliographystyle{icml2018}
\bibliography{ref}

\appendix

\onecolumn

\input{high-density-case}
\input{general-case}
\input{smallk}
\input{put-together}
\input{lower_bound}
\input{sieve_lower_bound}
\input{p-pass}

\input{remove_opt}

\end{document}

%% file: 050-abstract.tex
\begin{abstract}
Many tasks in machine learning and data mining,
such as
data diversification,
non-parametric learning,
kernel machines,
clustering
etc.,
require
extracting a small but representative
summary
from a
massive
dataset.
Often, such problems
can be
posed as
maximizing a \emph{submodular} set function
subject to a cardinality constraint.
We consider this question in the \emph{streaming setting},
where elements arrive
over time at a fast pace
and
thus
we need to design an efficient, low-memory algorithm.
One such method,
proposed by
Badanidiyuru et al.~(2014),
always finds
a $0.5$-approximate solution.
Can this approximation factor be improved?
We answer this question affirmatively by
designing a new algorithm \ouralgo
for streaming submodular maximization.
It is the first low-memory, single-pass algorithm
that improves the factor $0.5$,
under the natural assumption that elements arrive in a random order.
We also show that this assumption is necessary,
i.e.,
that there is no such algorithm
with better than $0.5$-approximation
when elements arrive in arbitrary order.
Our experiments demonstrate that \ouralgo
significantly outperforms the state of the art
in applications related to
exemplar-based clustering,
social graph analysis,
and
recommender systems.
\end{abstract} 

%% file: 100-introduction.tex
\section{Introduction}

We are experiencing
an unprecedented growth
in the sizes of
modern
datasets.
Streams of data
of massive volume
are generated every second,
coming from many
different
sources
in industry and science
such as:
image and video streams,
sensor data,
social networks,
stock markets,
and many others.
Sometimes,
such data is produced so rapidly
that most of it cannot even be stored
in any way.
In this context,
a
critical task
is
\emph{data summarization}:
one of
extracting a representative subset of manageable size
from
rich,
large-scale data streams.
A central topic in machine learning and data mining today,
its
main
challenge is to produce
such a concise yet high-value summary
while doing so
efficiently
and
on-the-fly.

In many applications,
this challenge can be viewed
as optimizing a \emph{submodular function}
subject to a cardinality constraint.
Indeed, submodularity
-- an intuitive notion of diminishing returns,
which postulates that an element should contribute more to a smaller set than to a larger one --
plays a similar role in this setting
as convexity does in continuous optimization.
Namely,
it is general enough to model a multitude of practical scenarios,
such as
viral marketing \cite{kempe2003maximizing},
recommender systems \cite{elarini2011beyond},
search result diversification \cite{agrawal2009diversifying}
or active learning \cite{golovin2011adaptive},
while allowing for both theoretically and practically sound and efficient algorithms.
In particular, a celebrated result by Nemhauser et al.~\yrcite{nemhauser1978analysis}
shows that the \greedy algorithm
-- one that iteratively picks the element with the largest marginal contribution to the current summary --
is a $\rb{1 - \nicefrac{1}{e}}$-approximation
for maximizing a monotone submodular function
subject to a cardinality constraint.
That is,
the objective value that it attains
is at least
a $\rb{1-\nicefrac{1}{e}}$-fraction
of the optimum.
This approximation factor is NP-hard to improve~\cite{feige1998threshold}.
Unfortunately,
\greedy requires repeated
access to the complete dataset,
which precludes it from
use in
large-scale applications
in terms of both memory and running time.

The sheer bulk of large datasets
and
the infeasibility of \greedy
together imply
a growing need for faster
and memory-efficient algorithms,
ideally ones that can work
in the \emph{streaming setting}:
one where input arrives one element at a time,
rather than being available all at once,
and only a small portion of the data
can be kept in memory at any point.
The first such algorithm
was given by Chakrabarti and Kale~\yrcite{chakrabarti2014submodular},
yielding
a $0.25$-approximation
while requiring only a single pass over the data,
in arbitrary order,
and using
$O(k)$ function evaluations per element
and
$O(k)$ memory.\footnote{
	We make the usual assumption that one can store any element,
	or the value of any set,
	using $O(1)$ memory.
	The memory usage calculation in \cite{chakrabarti2014submodular}
	is
	lower-level,
	which results in an extra $\log n$ factor.
	Furthermore,
	their algorithm can be implemented using a priority queue,
	which would result in a runtime of $O(\log k)$ per element.
}
A more accurate and efficient method
\sieve
was proposed by Badanidiyuru et al.~\yrcite{badanidiyuru2014streaming}.
For any $\eps > 0$,
it
provides a $(0.5 - \eps)$-approximation
and uses
$O\rb{\nicefrac{\log k}{\eps}}$ function evaluations per element
and
$O\rb{\nicefrac{k \log k}{\eps}}$ memory.
While well-suited for use
in big data stream scenarios,
its approximation guarantee is nevertheless
still
inferior
to that of \greedy.
It is natural to wonder:
can the
ratio $0.5$
be improved upon
by a more accurate algorithm?

It turns out that
in general,
the answer is no (modulo the natural assumption that the submodular function is only evaluated on feasible sets).
As one of our results,
we show that:
\begin{theorem} \label{thm:hardness}
Any algorithm for streaming submodular maximization
that only queries the value of the submodular function on feasible sets (i.e.,  sets of cardinality at most $k$) and is an $\alpha$-approximation
for a constant $\alpha > 0.5$
must  use $\Omega(n/k)$ memory,
where $n$ is the length of the stream.
\end{theorem}
This hardness includes randomized algorithms,
and applies even for \emph{estimating} the optimum value to within this factor,
without necessarily returning a solution (see \cref{sec:hardness} for the proof).\footnote{
	Moreover,
	note that \cref{thm:hardness} does \emph{not} follow
	from the work of Buchbinder et al.~\yrcite{buchbinder2015online},
	who proved an approximation hardness of $0.5$
	for \emph{online} algorithms
	whose memory state must always be a feasible solution
	(consisting of at most $k$ elements).
}
Note that usually $n/k \gg k$;
such an algorithm therefore cannot run
in a large-scale streaming setting.

However,
this
bound pertains to
\emph{arbitrary-order} streams.
An immediate question, then,
is whether 
inherent randomness present in the stream can be helpful
in obtaining
higher accuracy.
Namely,
in many real-world scenarios
the data arrives
in random order,
or can be processed in random order.
This can be seen as a sweet spot between
assuming that the data is drawn randomly from an underlying prior distribution
-- which is usually unrealistic --
and needing to allow for instances whose contents and order are both adversarially designed
-- which also do not appear in applications.
Is it possible to obtain a better approximation ratio
under this natural assumption?

Again, we begin with a negative result:
the performance of
the state-of-the-art
\sieve algorithm
remains capped at $0.5$
in this setting.
\begin{theorem} \label{thm:sieve_half}
There exists a family of instances
on which the approximation ratio of \sieve
is at most $0.5 + o(1)$
even if elements arrive in random order.
\end{theorem}
We remark that
\cref{thm:sieve_half} also extends to
certain
natural modifications of \sieve
(with a different value of
a threshold parameter used in the algorithm,
or multiple such values
that are
tried in parallel).
Thus,
new ideas are required to go beyond
an approximation ratio of $0.5$.

As the main result of this paper
we present a new algorithm \ouralgo (Streaming ALgorithm for Submodular maximization with Adaptive thresholding),
which does break the $0.5$ barrier
in the random-order case.
\ouralgo, like \sieve,
works in the streaming model
and takes only a single pass over the data,
selecting those elements whose marginal contribution
is above some current threshold.
However,
it employs an \emph{adaptive}
thresholding scheme,
where the threshold is chosen dynamically
based on the objective value
obtained until a certain
point in the data stream.
This additional power allows us to prove the following guarantee:
\begin{restatable}{theorem}{maintheorem}[Main Theorem]\label{main-theorem}
There exists a constant $\alpha > 0.5$
such that,
for any stream of elements that arrive in random order,
the value of the solution returned by \ouralgo
is at least $\alpha \cdot \fopt$ in expectation
(where $\fopt$ is the value of the optimum solution).
\ouralgo uses $O(k \log k)$ memory
(independent of the length of the stream)
and processes each element using $O(\log k)$
evaluations of the objective function.
\end{restatable}
We remark that
even if the stream is adversarial-order,
\ouralgo still guarantees a $(0.5 - \eps)$-approximation. 

A different way to improve the accuracy of an algorithm is allowing it to make multiple passes over the stream. In this paper we also consider this setting and present a $2$-pass algorithm \twopass for streaming submodular maximization. We show that \twopass achieves a $\nicefrac{5}{9}$ approximation ratio using the same order of memory and function evaluations as \sieve.
Formally, for any $\eps > 0$ we show that:

\begin{theorem}
\twopass is a $\rb{\nicefrac{5}{9}-\eps}$-approximation for streaming submodular maximization. It uses $O\rb{\nicefrac{k \log k}{\eps}}$ memory and processes each element with $O\rb{\nicefrac{\log k}{\eps}}$ evaluations of the objective function.
\end{theorem}

Furthermore, we generalize our ideas to design a $p$-pass algorithm \ppass for any constant $p \ge 2$. McGregor and Vu~\yrcite{mcgregor2016better} showed that, regardless of the running time, no $p$-pass algorithm can beat the $(1-\nicefrac{1}{e})$ approximation guarantee using memory $\poly(k)$. In this work we show that \ppass quickly converges to a $(1-\nicefrac{1}{e})$-approximation as $p$ grows. 
We show that: 
\begin{theorem}
\ppass is a $\rb{1-\rb{\nicefrac{p}{p+1}}^p-\eps}$-approximation for streaming submodular maximization. It uses $O\rb{\nicefrac{k \log k}{\eps}}$ memory and processes each element with $O\rb{\nicefrac{p\log k}{\eps}}$ evaluations of the objective function.
\end{theorem} 



\paragraph{Applications and experiments}

We assess the accuracy of our algorithms and show their versatility in several real-world scenarios.
In particular, we study maximum coverage in graphs, exemplar-based clustering, and personalized movie recommendation. We find that \ouralgo significantly outperforms the state of the art, \sieve,
in all tested datasets. In fact, our experimental results show that \ouralgo reduces the gap between \greedy, which is the benchmark algorithm even for the offline setting (a ``tractable optimum''), and the best known streaming algorithm by a factor of two on average.

Note that
we are able to obtain these practical improvements
even though,
in our experiments,
the order of arrival of elements
is
\emph{not}
manually randomized.
This suggests that
the random-order assumption,
which allows us to obtain our improved theoretical guarantees,
does well in approximating
the nature of
real-world datasets,
which are not stochastic but also not adversarial.

\paragraph{Related work}

The benchmark algorithm for
monotone submodular maximization under a cardinality constraint
is \greedy.
Unfortunately, it is not efficient
and requires $k$ passes over the entire dataset.
There has thus been much interest
in obtaining more efficient versions
of \greedy,
such as \textsc{Lazy-Greedy}~\cite{minoux1978accelerated,leskovec2007costeffective,krause2008nearoptimal},
the algorithm of Badanidiyuru and Vondrák~\yrcite{badanidiyuru2014fast},
or \textsc{Stochastic-Greedy}~\cite{mirzasoleiman2015lazier}.

The first \emph{multi-pass} algorithm
for {streaming} submodular maximization
has been given by Gomes and Krause~\yrcite{gomes10budgeted}.
If $f$ is upper-bounded by $B$, then for any $\eps > 0$ their algorithm attains the value $0.5 \cdot \fopt - k \eps$ and uses $O(k)$ memory while making $O(B/\eps)$ passes.
Interestingly, it converges to the optimal solution for a restricted class of submodular functions.

Many different settings are considered under the streaming model.
One important requirement often arising in practice
is that the returned solution be
robust against deletions~\cite{mirzasoleiman2017deletion,mitrovic2017streaming,kazemi2017deletion}.
Non-monotone submodular functions have also been considered~\cite{chekuri2015streaming,mirzasoleiman2017streaming}.

McGregor and Vu~\yrcite{mcgregor2016better}
consider the $k$-coverage problem in the multi-pass streaming setting.
They give an algorithm achieving $(1-\nicefrac{1}{e}-\eps)$-approximation
in $O(1/\eps)$ passes.
They also show that improving upon the ratio $(1-\nicefrac{1}{e})$
in a constant number of passes would require an almost linear memory.
Their results generalize from $k$-coverage to submodular maximization.

In the \emph{online} setting,
the stream length $n$ is unknown
and the algorithm must always maintain a feasible solution.
This model allows preemption, i.e., the removal of previous elements from the solution
(otherwise no constant competitive ratio is possible).
Chakrabarti and Kale~\yrcite{chakrabarti2014submodular},
Chekuri et al.~\yrcite{chekuri2015streaming}
and Buchbinder et al.~\yrcite{buchbinder2015online}
have obtained $0.25$-competitive algorithms
for monotone submodular functions
under a cardinality constraint.
This competitive ratio was later
improved to $0.29$ by Chan et al.~\yrcite{chan2017online}.
Buchbinder et al.~\yrcite{buchbinder2015online}
also prove
a hardness of $0.5$.

A different large-scale scenario is the \emph{distributed} one,
where the elements are partitioned across $m$ machines.
The algorithm \textsc{GreeDi}~\cite{mirzasoleiman2013distributed}
consists in running \greedy
on each machine and then combining the resulting summaries
on a single machine using another run of \greedy.
This yields an $O\rb{1/{\min(\sqrt{k}, m)}}$-approximation.
Barbosa et al.~\yrcite{barbosa2015power}
showed that when the elements are partitioned \emph{randomly},
one obtains a $\rb{1-\nicefrac{1}{e}}/2$-approximation.
Mirrokni and Zadimoghaddam~\yrcite{mirrokni2015randomized}
provide a different two-round strategy:
they
compute \emph{coresets} of size $O(k)$
and then greedily merge them,
yielding a $0.545$-approximation.
The algorithm
of Kumar et al.~\yrcite{kumar2015fast}
consists of a logarithmic number of rounds in the MapReduce setting
and approaches the \greedy ratio.
Barbosa et al.~\yrcite{barbosa2016new}
provide a general reduction
that implies a $(1-\nicefrac{1}{e}-\eps)$-approximation
in $O(1/\eps)$ rounds.

Korula et al.~\yrcite{korula2015online}
study the Submodular Welfare Maximization problem --
where a set of items needs to be partitioned among agents in order to maximize social welfare,
i.e., the sum of the (monotone submodular) utility functions of the agents
-- in the online setting.
The best possible competitive ratio in general is $0.5$.
However, they show that the greedy algorithm is $0.505$-competitive
if elements arrive in random order.


%% file: 200-preliminaries.tex
\section{Preliminaries}
We consider a (potentially large) collection $V$ of $n$ items, also called the \emph{ground set}. We study the problem of maximizing a \emph{non-negative monotone submodular function} $f : 2^V \to \bRplus$. Given two sets $X, Y \subseteq V$, the \emph{marginal gain} of $X$ with respect to $Y$ is defined as
\[
	\marginal{X}{Y} = f(X \cup Y) - f(Y) \,,
\]
which quantifies the increase in value when adding $X$ to $Y$. We say that $f$ is \emph{monotone} if for any element $e \in V$ and any set $Y \subseteq V$ it holds that $\marginal{e}{Y} \ge 0$. The function $f$ is \emph{submodular} if for any two sets $X$ and $Y$ such that $X \subseteq Y \subseteq V$ and any element $e \in V \setminus Y$ we have
\[
	\marginal{e}{X} \ge \marginal{e}{Y}.
\]
Throughout the paper, we assume that $f$ is given in terms of a value oracle that computes $f(S)$ for given $S \subseteq V$. We also assume that $f$ is \emph{normalized}, i.e. $f(\emptyset) = 0$.

\subsubsection*{Submodularity under cardinality constraint}

The problem of maximizing function $f$ under \emph{cardinality constraint $k$} is defined as
selecting a set $S \subseteq V$
with $|S| \le k$
so as to maximize $f(S)$.
We will use $\opt$ to refer to such a set $S$, $\fopt$ to denote $f(\opt)$, and the name $\SEXYP$ to refer to this problem.

%% file: 300-algorithm.tex
\input{overviewofalg.tex}
\input{2pass.tex}

%% file: overviewofalg.tex
\section{Overview of \ouralgo}
\label{sec:overview-alg}
In this section, we present an overview of our algorithm. We also explain the main ideas and the key ingredients of its analysis. In Appendix~\ref{sec:putting-together}, we combine these ideas into a proof of Theorem~\ref{main-theorem}. Throughout this section, we assume that the value OPT of an optimal solution $\cO=\{o_1,...,o_k\}$ is known in advance. We show how to remove this assumption using standard techniques in \cref{sec:removing-opt}.


We start by defining the notion of {\emph{dense}} optimum solutions. We say that $\cO$ is dense if there exists a set $D \subseteq \cO$ of size at most $k/100$ such that $f(D) \geq OPT/10$.\footnote{In the appendix, we slightly alter the constants in the definition of a dense optimal solution.}
Our algorithm runs three procedures, and each procedure outputs a set of at most $k$ elements. One of the procedures performs well in the case when $\cO$ is dense. The other two approaches are designed to collect high utility when $\cO$ is not dense. We run these procedures in parallel and, out of the three returned sets, we report the one attaining the highest utility.
In what follows, we first describe our algorithm for the case when $\cO$ is not dense.

\paragraph{Case: $\cO$ is not dense.}
We present the intuition behind the algorithm under the simplifying assumption that $f(o)=\text{OPT}/k$ for every $o \in \cO$. However, the algorithm that we state provides an approximation guarantee better than $0.5$ in expectation for \emph{any} instance that is not dense.

Over the first $0.1$-fraction of the stream, both procedures for this case behave identically: they maintain a set of elements $S$; initially, $S=\emptyset$; each element $e$ from the stream is added to $S$ if its marginal gain is at least $T_1=\frac{\text{OPT}}{k}(\nicefrac{1}{2}+\epsilon)$, i.e.,
\[	
	f(e|S) \geq \frac{\text{OPT}}{k}(\nicefrac{1}{2}+\epsilon).
\]
Consider the first element $o \in \cO$ that the procedures encounter on the stream. Since the stream is randomly ordered, $o$ is a random element of $\cO$. Due to this, we claim that if $f(S)$ is small, then it is likely that the procedures add $o$ to $S$. This follows from the fact that each element of $\cO$ is worth OPT$/k$. Namely, if $f(S) < \text{OPT}(\nicefrac{1}{2}-\epsilon')$, for a small constant $\epsilon' > 0$, then the average marginal contribution of the elements of $\cO$ with respect to $S$ is more than $T_1$, hence it is likely that the procedures select $o$. By repeating the same argument we can conclude that after processing a $0.1$-fraction of the stream, either: (1) $f(S)$ is large, i.e., $f(S) > \text{OPT}(\nicefrac{1}{2}-\epsilon')$; or (2) the procedures have selected $k / 100$ elements from $\cO$ (which are worth $\text{OPT}/100$).

Up to this point, both procedures for the non-dense case behaved identically. 
In the remaining $0.9$-fraction of the stream, the procedure corresponding to case~(1) above uses a threshold $\frac{\text{OPT}}{k}(\nicefrac{1}{2}-\delta)$, which is lower than $T_1$. Since there are still $0.9n$ elements left on the stream, and already after the first $0.1$-fraction we have $f(S) > \text{OPT}(\nicefrac{1}{2}-\epsilon')$, it is very likely that by the end the procedure will have added enough further elements to $S$ so that $f(S)\geq \text{OPT}(\nicefrac{1}{2}+\epsilon)$. 

In case~(2) above, the procedure has already selected a set $S$ that contains at least $k/100$ elements from $\cO$, i.e., $|S\cap \cO|\geq k/100$.  Now, the procedure corresponding to this case continues with the threshold $T_1 = \frac{\text{OPT}}{k}(\nicefrac{1}{2}+\epsilon)$. If by the end of the stream the procedure has selected $k$ elements, then clearly $f(S) \geq \text{OPT}(\nicefrac{1}{2}+\epsilon)$, since each element has marginal gain at least $T_1$. Otherwise, the procedure has selected fewer than $k$ elements. This means that the marginal gain of any element of the stream with respect to $S$ is less than $T_1$. 
Now we claim that $f(S)>\text{OPT}/2$. First, there are at most $99k/100$ elements in $\cO \setminus S$. Furthermore, adding each such element to the set $S$ gives marginal gain less than $T_1$. Therefore, the total benefit that the elements of $\cO \setminus S$ give to $S$ is at most $\frac{\text{OPT}}{k}(\nicefrac{1}{2}+\epsilon)\cdot 99k/100$, which is less than $\text{OPT}(\nicefrac{1}{2}-\nicefrac{1}{300})$ for small enough $\epsilon$, therefore
\[ \text{OPT}  \leq f(S \cup \cO) = f(S) + f(\cO | S) \]
and thus
\begin{align*}
f(S)&>\text{OPT}-\text{OPT}(\nicefrac{1}{2}-\nicefrac{1}{300}) \\&= \text{OPT}(\nicefrac{1}{2}+\nicefrac{1}{300}).
\end{align*}

\paragraph{Case: $\cO$ is dense.} We now give a brief overview of the procedure that is designed for the case when $\cO$ is dense. Over the first $0.8$-fraction of the stream, the procedure uses a (high) threshold $T'_1=\frac{\text{OPT}}{k}\cdot 2$. Let $D \subseteq \cO$ be the dense part of $\cO$. Note that the average value of the elements of $D$ is at least $\frac{\text{OPT}}{k}\cdot 10$, which is significantly higher than the threshold $T'_1$.

Hence, even over the $0.8$-fraction of the stream, the algorithm will in expectation collect some elements with large marginal gain. This, intuitively, means that the algorithm in expectation selects $k'$ elements of total value significantly larger than $k' \text{OPT} / (2 k)$. This enables us to select the remaining $k - k'$ elements with  marginal gain below $\text{OPT} / (2 k)$ and still collect a set of utility larger than $\text{OPT}/2$. We implement this observation by letting the algorithm use a threshold lower than $\text{OPT}/(2 k)$ for the remaining $0.2$-fraction of the stream. This increases the chance that the algorithm collects $k - k'$ more elements.

In what follows, we provide pseudo-codes of our three algorithms. For sake of brevity, we fix the values of constants and give the full analysis of the algorithms in Appendix~\ref{mainalgo}.

We begin with the dense case, presented in Algorithm~\ref{algo1}. In the pseudo-code, $C_1, C_2$ are large absolute constants and $\beta$ is the fraction of the stream that we process with a high threshold. 
\begin{algorithm}[h!]
\begin{algorithmic}[1]
\caption{\label{algo1} \dense}

	\State $S:=\emptyset$
	\For{the $i$-th element $e_i$ on the stream}
		\If  {$i\leq \beta n$ and $f(e_i|S) \geq \frac{C_1}{k} \text{OPT}$ and $|S|<k$}
			\State $S:=S\cup \{e_i\}$
		\ElsIf {$i> \beta  n$ and $f(e_i|S) \geq \frac{1}{C_2\cdot k} \text{OPT}$ and $|S|<k$}
			\State $S:=S\cup \{e_i\}$
		\EndIf  
    \EndFor\\
    \Return $S$

\end{algorithmic}
\end{algorithm}

For the case when $\cO$ is not dense, we use two algorithms as described above. The first algorithm (Algorithm~\ref{algo2}) goes over the stream and selects any element whose marginal gain to the currently selected elements is at least $\frac{\text{OPT}}{k}(\nicefrac{1}{2}+\epsilon)$. The second algorithm (Algorithm~\ref{algo3}) starts with the same threshold, but after passing over $\beta n$ elements it decreases the threshold to $\frac{\text{OPT}}{k}(\nicefrac{1}{2}-\delta)$.
\begin{algorithm}[h!]
\begin{algorithmic}[1]
\caption{\label{algo2} \fixed}
	\State $S:=\emptyset$
	\For{the $i$-th element $e_i$ on the stream }
		\If  {$f(e_i|S) \geq \frac{\text{OPT}}{k}(\nicefrac{1}{2}+\epsilon)$ and $|S|<k$}
			\State $S:=S\cup \{e_i\}$
		\EndIf  
    \EndFor\\
    \Return $S$
\end{algorithmic}
\end{algorithm}

\begin{algorithm}[h!]
\caption{\label{algo3} \highlow}
\begin{algorithmic}[1]
	\State $S:=\emptyset$
	\For{the $i$-th element $e_i$ on the stream }
		\If  {$i\leq \beta  n$ and $f(e_i|S) \geq \frac{\text{OPT}}{k} (\nicefrac{1}{2}+\epsilon)$ and $|S|<k$}
			\State $S:=S\cup \{e_i\}$
		\ElsIf {$i> \beta  n$ and $f(e_i|S) \geq \frac{\text{OPT}}{k} (\nicefrac{1}{2}-\delta)$ and $|S|<k$}
			\State $S:=S\cup \{e_i\}$
		\EndIf  
    \EndFor\\
    \Return $S$

\end{algorithmic}
\end{algorithm}
Since we do not know in advance whether the input is dense or not, we run these three algorithms in parallel and output the best solution at the end. 

%% file: 2pass.tex
\section{\twopass Algorithm}
In this section, we describe our \twopass algorithm. Recall that we denote the optimum solution by $\cO=\{o_1,\ldots,o_k\}$ and we let $\text{OPT}=f(\cO)$. Throughout this section, we assume that OPT is known. We show how to remove this assumption in \cref{sec:removing-opt}. Also, in Appendix~\ref{sec:multi-pass} we present a (more general) $p$-pass algorithm.

Our \twopass algorithm (Algorithm~\ref{algo2pass}) is simple: in the first pass we pick any element whose marginal gain with respect to the currently picked elements is higher than the threshold $T_1=\frac{2}{3} \cdot \frac{\text{OPT}}{k}$. In the second pass we do the same using the threshold $T_2=\frac{4}{9} \cdot \frac{\text{OPT}}{k}$. 

\begin{algorithm}[h!]
\caption{\label{algo2pass} \twopass Algorithm}
\begin{algorithmic}[1]
	\State $S:=\emptyset$
	\For{the $i$-th element $e_i$ on the stream}
		\If  {$f(e_i|S) \geq \frac{\text{2OPT}}{3k}$ and $|S|<k$}
			\State $S:=S\cup \{e_i\}$
		\EndIf  
    \EndFor
    	\For{the $i$-th element $e_i$ on the stream}
		\If  {$f(e_i|S) \geq \frac{\text{4OPT}}{9k}$ and $|S|<k$}
			\State $S:=S\cup \{e_i\}$
		\EndIf  
    \EndFor\\
    \Return $S$
\end{algorithmic}
\end{algorithm}

\begin{theorem}
\twopass is a $\nicefrac{5}{9}$-approximation for $\SEXYP$.
\end{theorem} 

\begin{proof}
We prove this theorem in two cases depending on $|S|$.
First we consider the case $|S|<k$. For any element $ o \in \cO \setminus S$  we have $f(o|S_i) \leq T_2$ since we have not picked it in the second pass. 
Therefore \[f(\cO|S) \le \sum_{o \in \cO} f(o|S) \leq k\cdot  T_2 = \nicefrac{4}{9} \cdot \text{OPT}.\] Thus
\[\text{OPT}  \leq f(S \cup \cO) = f(S) + f(\cO |S)\]
and so
\[
f(S) \geq \text{OPT}\cdot(1- \nicefrac{4}{9}) = \nicefrac{5}{9} \cdot \text{OPT}.
\]
Therefore in this case we get the desired approximation ratio.

Now we consider the second case, i.e., $|S|=k$. It is clear that if we have picked $k$ elements in the first round, then we get a $\nicefrac{2}{3}$-approximation guarantee. Therefore assume that we picked fewer than $k$ elements in the first round, and let $S_1$ denote these elements. With a similar argument as in the previous case we get that $f(S_1) \geq \text{OPT}/3$. One can see that in the worst-case scenario, in the first pass we have picked $k/2$ elements with marginal gain exactly $T_1$ each and in the second pass we have picked $k/2$ elements with marginal gain exactly $T_2$ each (we present a formal proof of this statement in the appendix). Therefore we have:
\begin{align*}
f(S) &\geq k/2 \cdot T_1 +k/2 \cdot T_2 \\
&= k/2 \cdot \frac{2}{3} \cdot \frac{\text{OPT}}{k} +k/2 \cdot \frac{4}{9} \cdot \frac{\text{OPT}}{k}  \\
&\geq \frac{\text{OPT}}{2} \cdot (\nicefrac{2}{3} + \nicefrac{4}{9}) \\
&\geq \nicefrac{5}{9} \cdot \text{OPT}\,.
\end{align*}
\end{proof}

%% file: 500-experiments.tex
\section{Empirical Evaluation}
In this section, we numerically validate our theoretical findings. Namely, we compare our algorithms, $\ouralgo$ and $\twopass$, with two baselines, $\greedy$ and $\sieve$. For this purpose, we consider three applications: (i)~dominating sets on graphs, (ii)~exemplar-based clustering, and (iii)~personalized movie recommendation. In each of the experiments we find that \ouralgo outperforms \sieve.

It is natural to consider the utility obtained by \greedy as a proxy for an optimum, as it is theoretically tight and difficult to beat in practice. The majority of our evaluations demonstrate that the gap between the solutions constructed by \ouralgo and \greedy is more than two times smaller than the gap between the solutions constructed by \sieve and \greedy. 

For each of the experiments we invoke our algorithms with the following values of the parameters: Algorithm~\ref{algo1} with $C_1 = 10$, $C_2 = 0.2$, $\beta = 0.8$; Algorithm~\ref{algo2} with $\eps = 1/6$; Algorithm~\ref{algo3} with $\beta = 0.1$, $\eps = 0.05$, $\delta = 0.025$.


\begin{figure*}[t!]
\newcommand{\spacebeforeincludegraphics}{-25pt}
\newcommand{\spaceafterincludegraphics}{-30pt}
\newcommand{\figurewidth}{1.13\linewidth}
\vspace{25pt}
\hspace{-30pt}
\minipage{0.34\textwidth}%
  \vspace*{\spacebeforeincludegraphics}
  \includegraphics[width=\figurewidth]{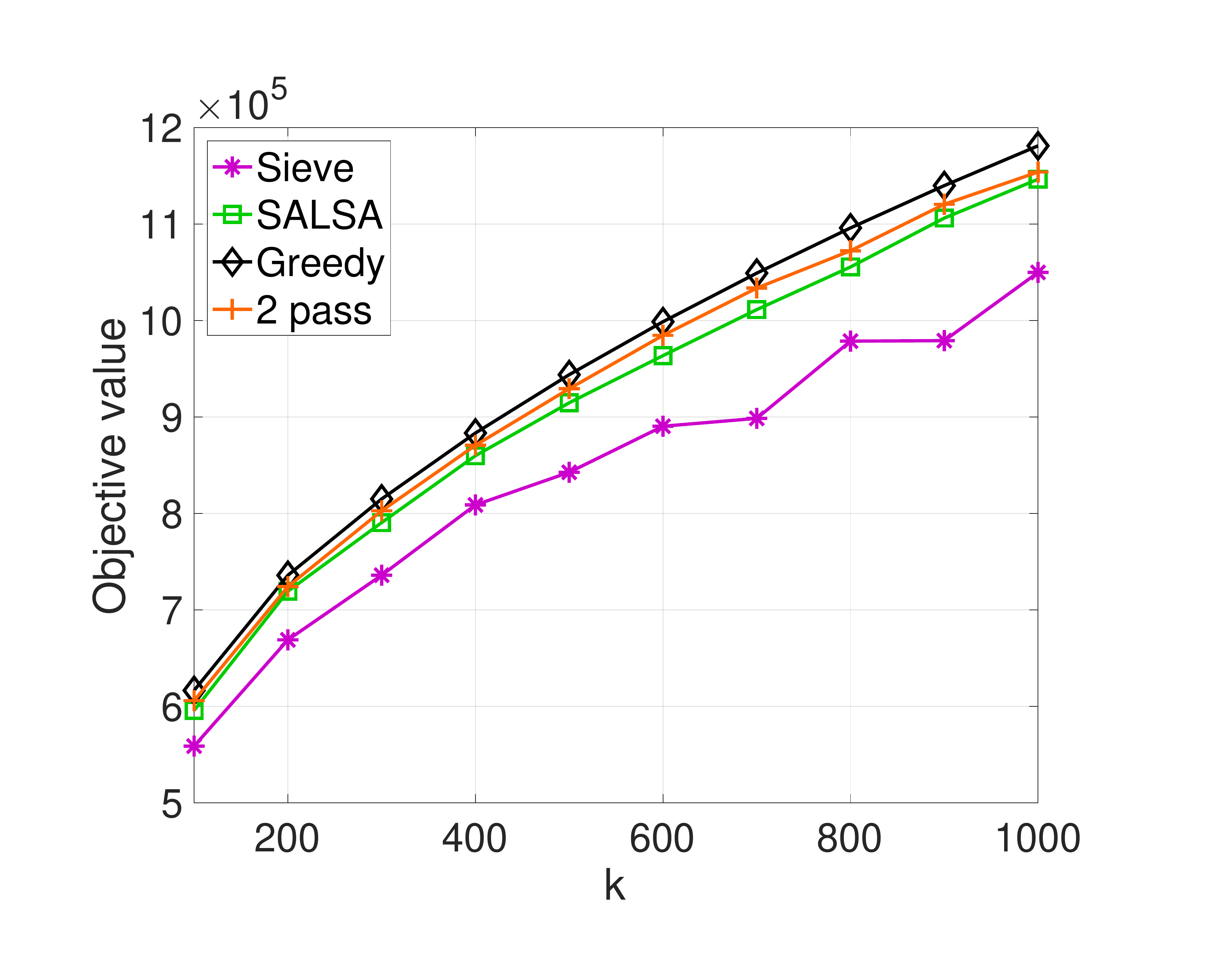}
  \vspace*{\spaceafterincludegraphics}
  \caption*{(a) Orkut}
\endminipage
\minipage{0.34\textwidth}%
  \vspace*{\spacebeforeincludegraphics}
  \includegraphics[width=\figurewidth]{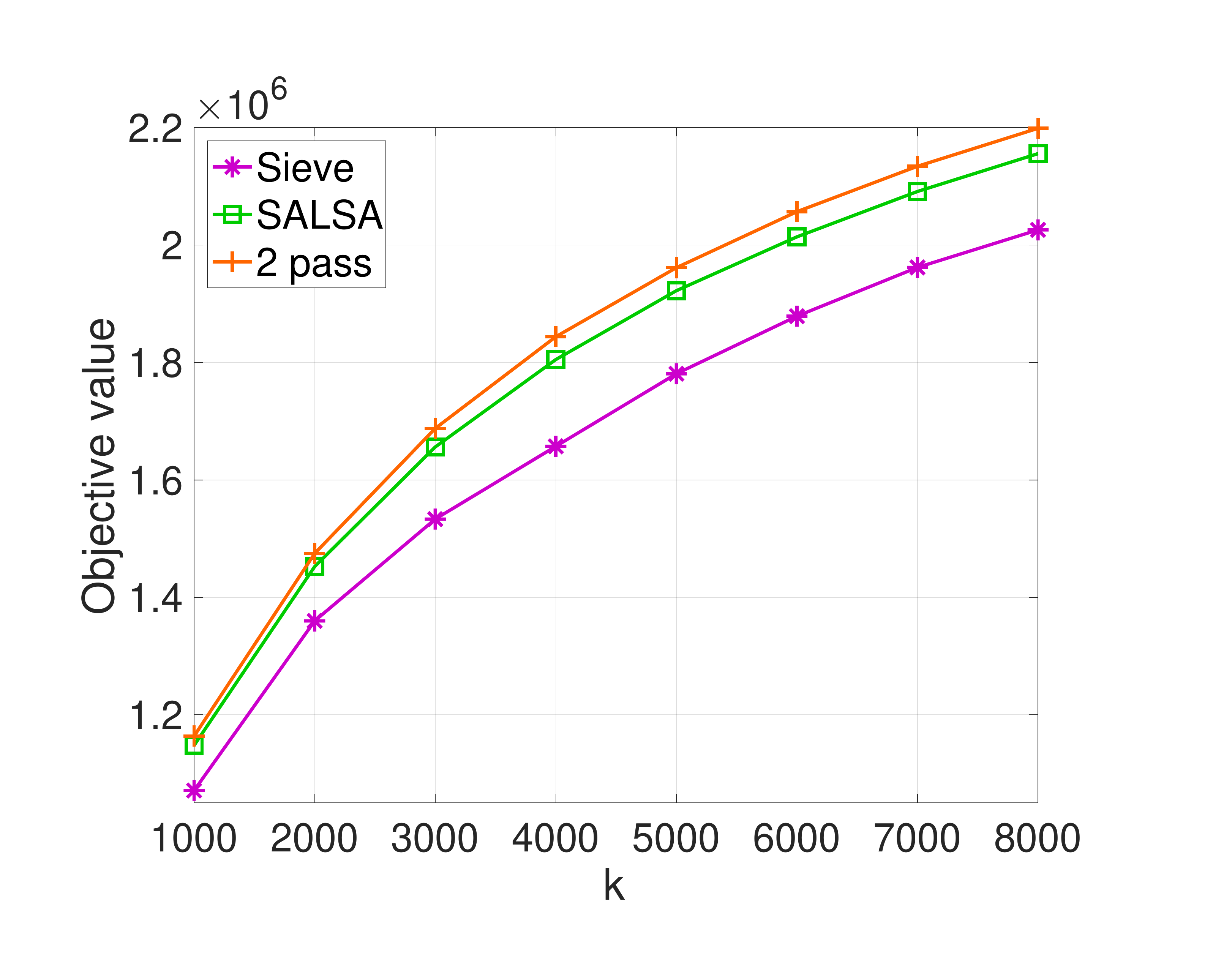}
  \vspace*{\spaceafterincludegraphics}
  \caption*{(b) Orkut}
\endminipage
\minipage{0.34\textwidth}
  \vspace*{\spacebeforeincludegraphics}
  \includegraphics[width=\figurewidth]{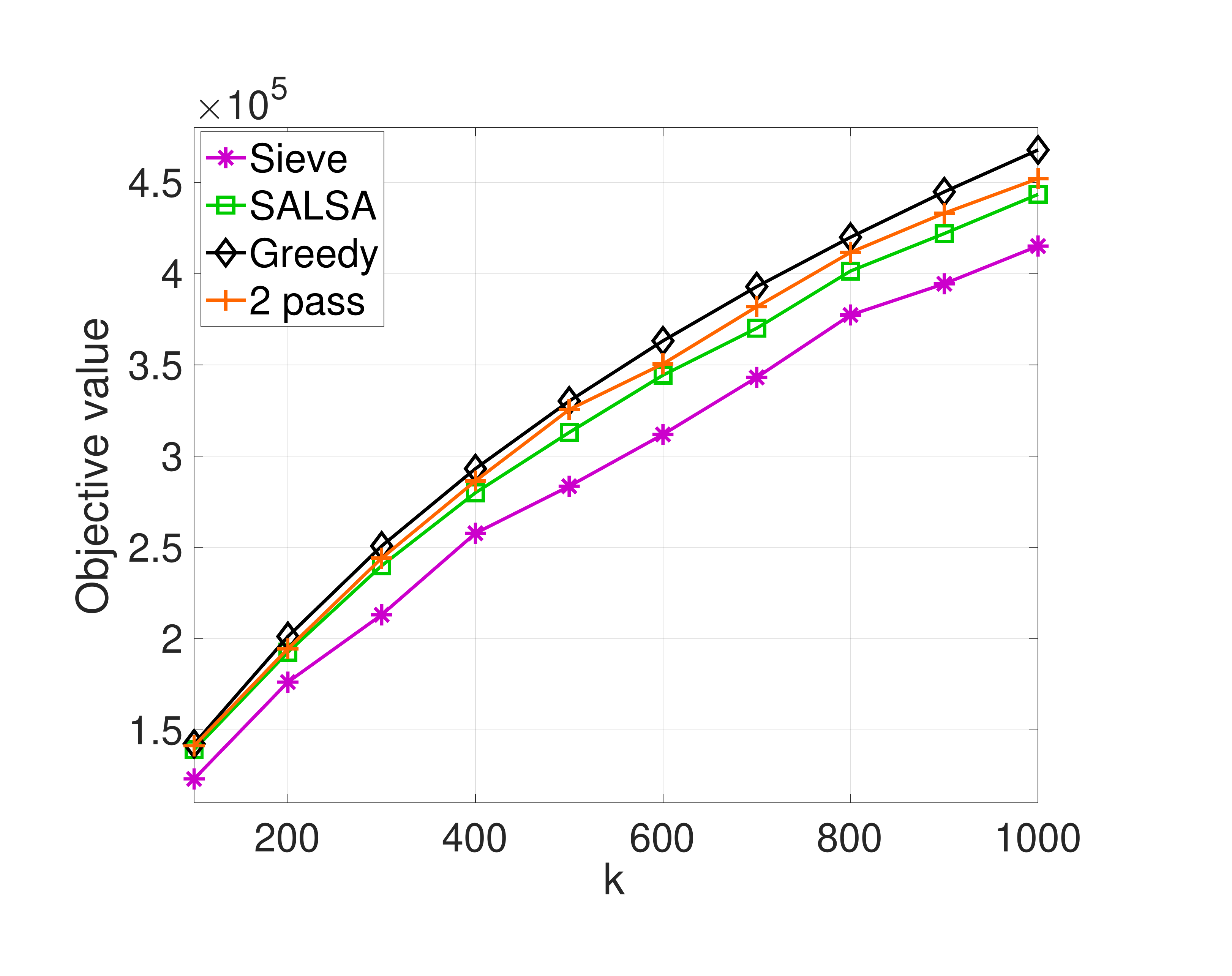}
  \vspace*{\spaceafterincludegraphics}
  \caption*{(c) LiveJournal}
\endminipage
\vspace{30pt}
\vspace{2mm}
\hspace{-30pt}
\minipage{0.34\textwidth}
  \vspace*{\spacebeforeincludegraphics}
  \includegraphics[width=\figurewidth]{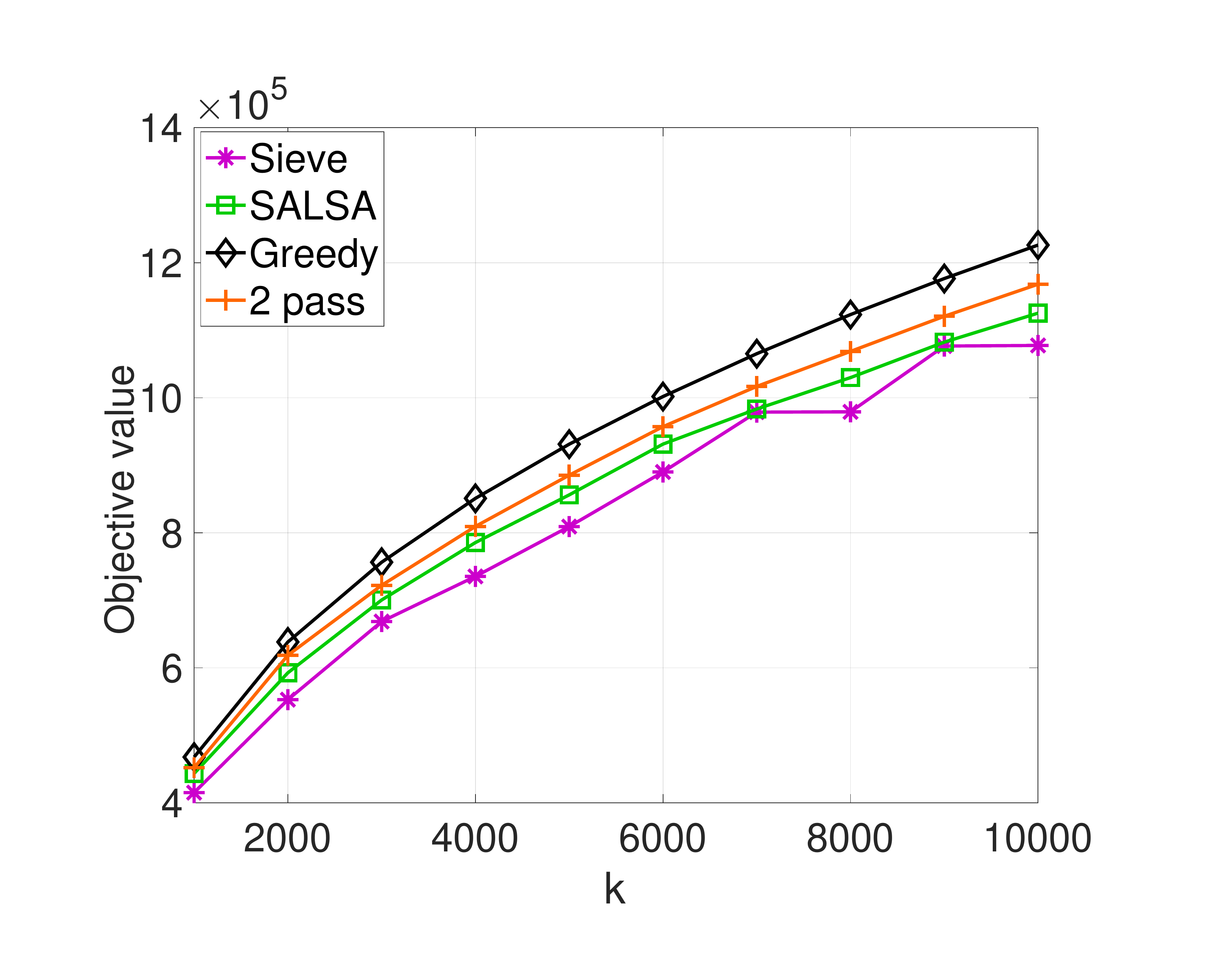}
  \vspace*{\spaceafterincludegraphics}
  \caption*{(d) LiveJournal}
\endminipage
\minipage{0.34\textwidth}%
  \vspace*{\spacebeforeincludegraphics}
  \includegraphics[width=\figurewidth]{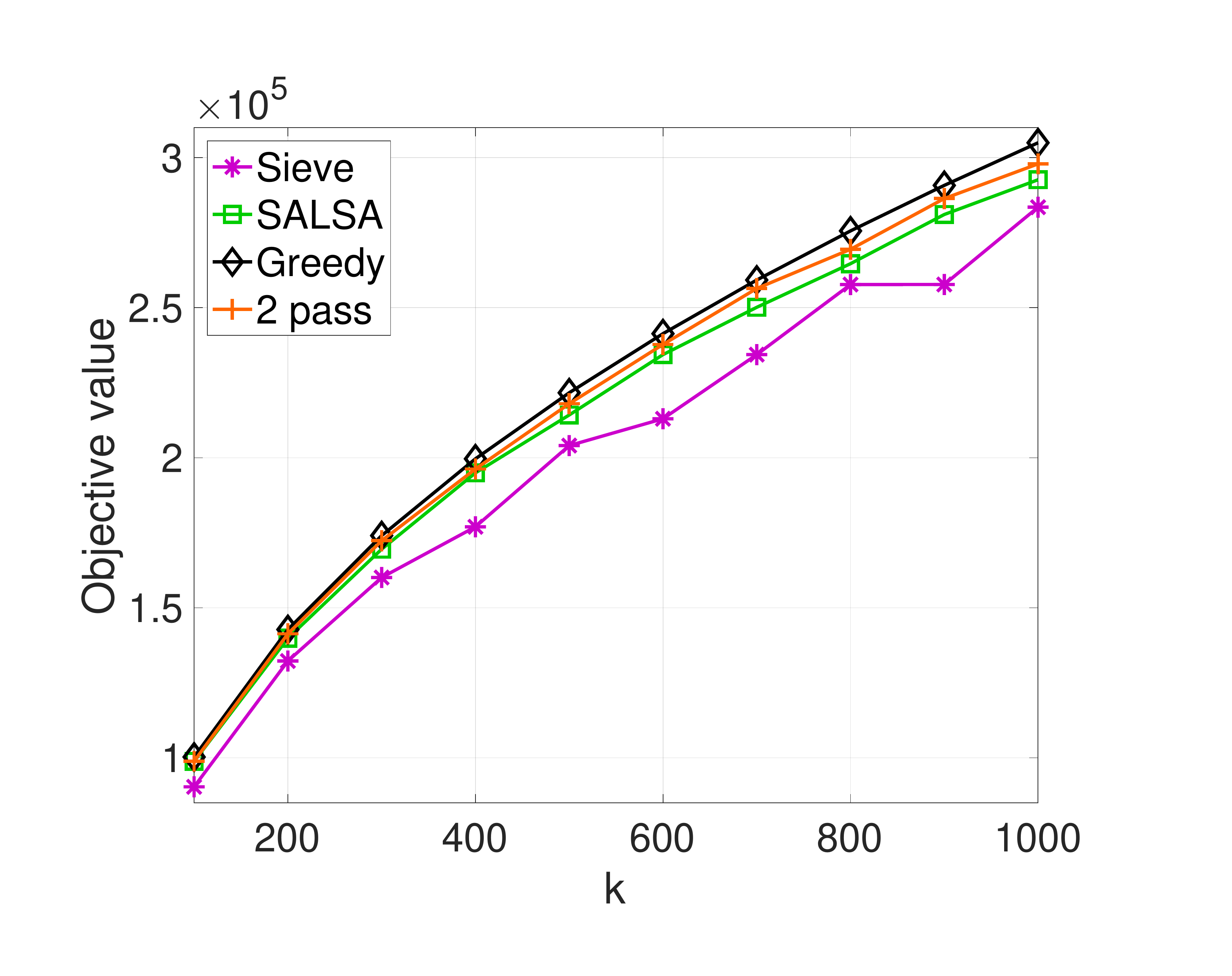}
  \vspace*{\spaceafterincludegraphics}
  \caption*{(e) Pokec}
\endminipage
\minipage{0.34\textwidth}%
  \vspace*{\spacebeforeincludegraphics}
  \includegraphics[width=\figurewidth]{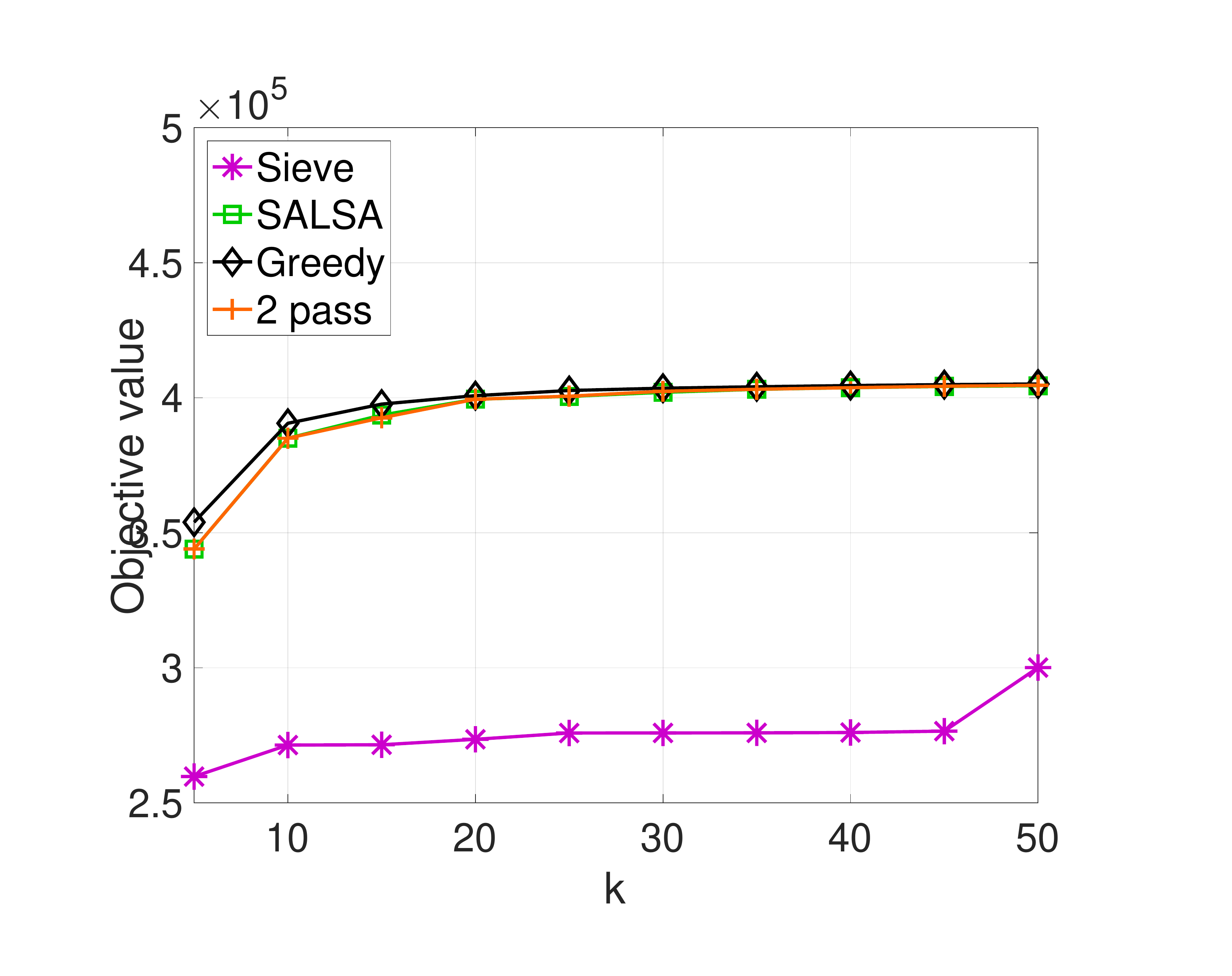}
  \vspace*{\spaceafterincludegraphics}
  \caption*{(f) Spambase}
\endminipage
\vspace{20pt}
\vspace{2mm}
\hspace{-12pt}
\minipage{0.34\textwidth}%
  \vspace*{\spacebeforeincludegraphics}
  \includegraphics[width=\figurewidth]{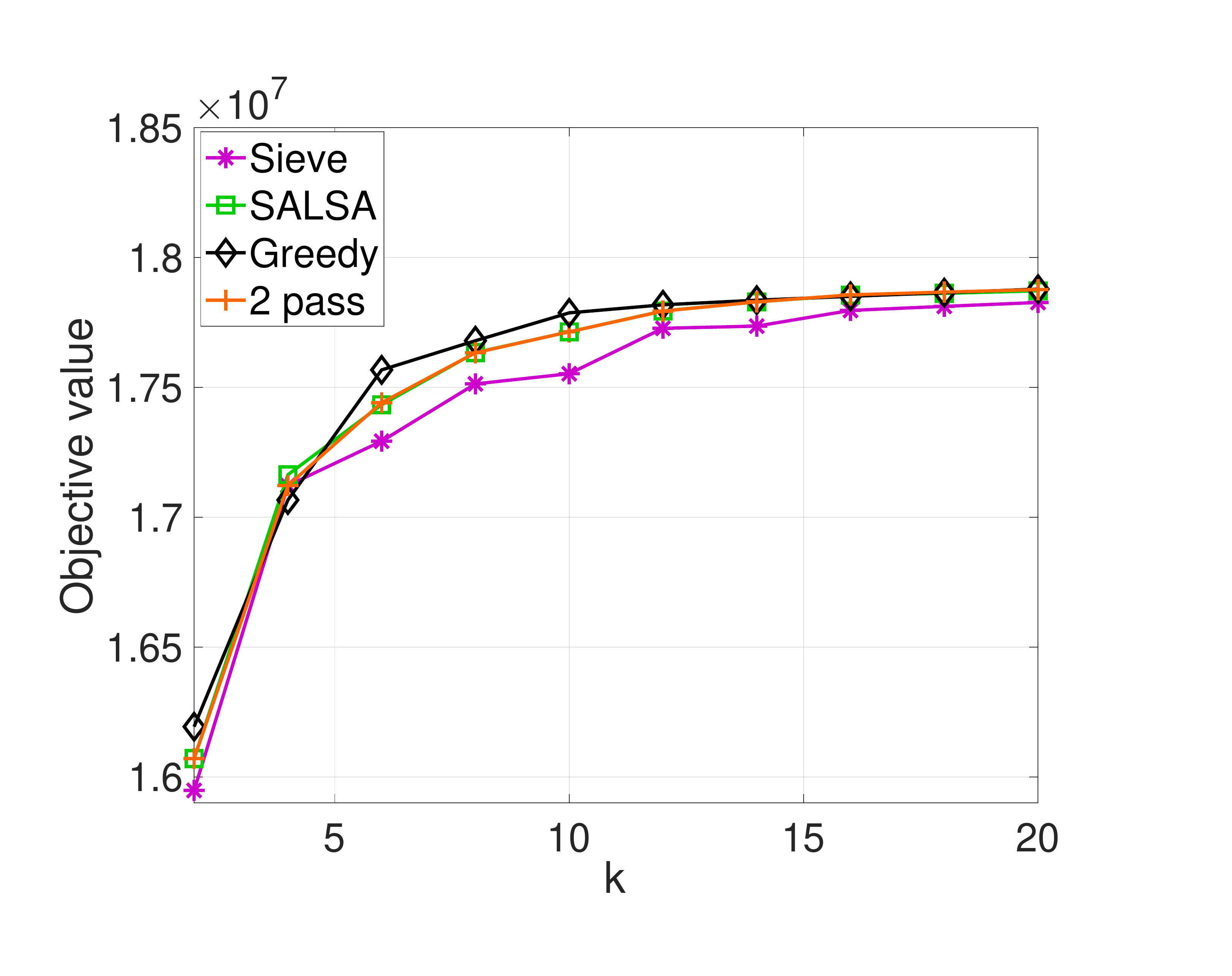}
  \vspace*{\spaceafterincludegraphics}
  \caption*{(g) CIFAR-10}
\endminipage
\minipage{0.34\textwidth}%
  \vspace*{\spacebeforeincludegraphics}
  \includegraphics[width=\figurewidth]{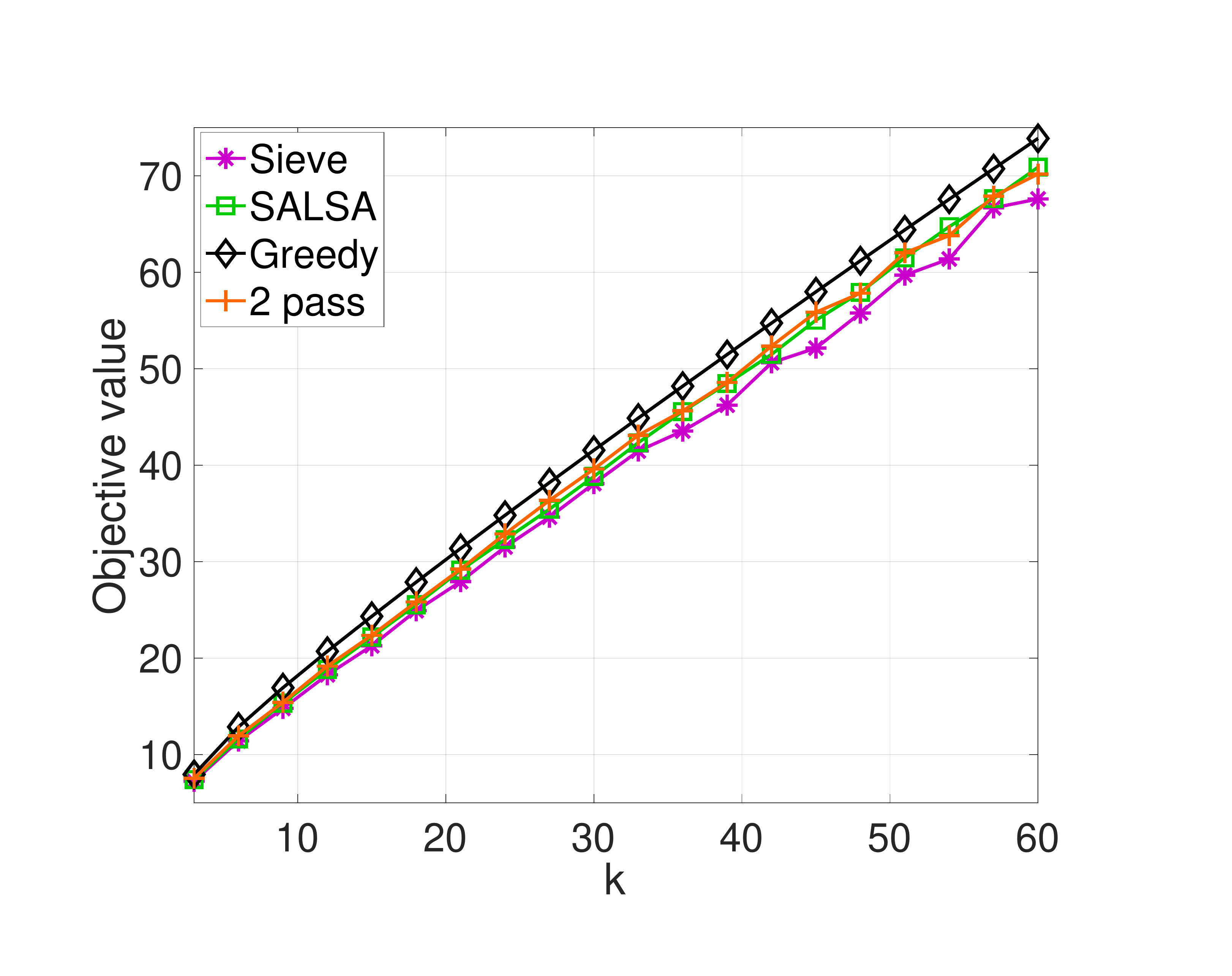}
  \vspace*{\spaceafterincludegraphics}
  \caption*{(h) Movies, $\alpha = 0.75$}
\endminipage
\minipage{0.34\textwidth}%
  \vspace*{\spacebeforeincludegraphics}
  \includegraphics[width=\figurewidth]{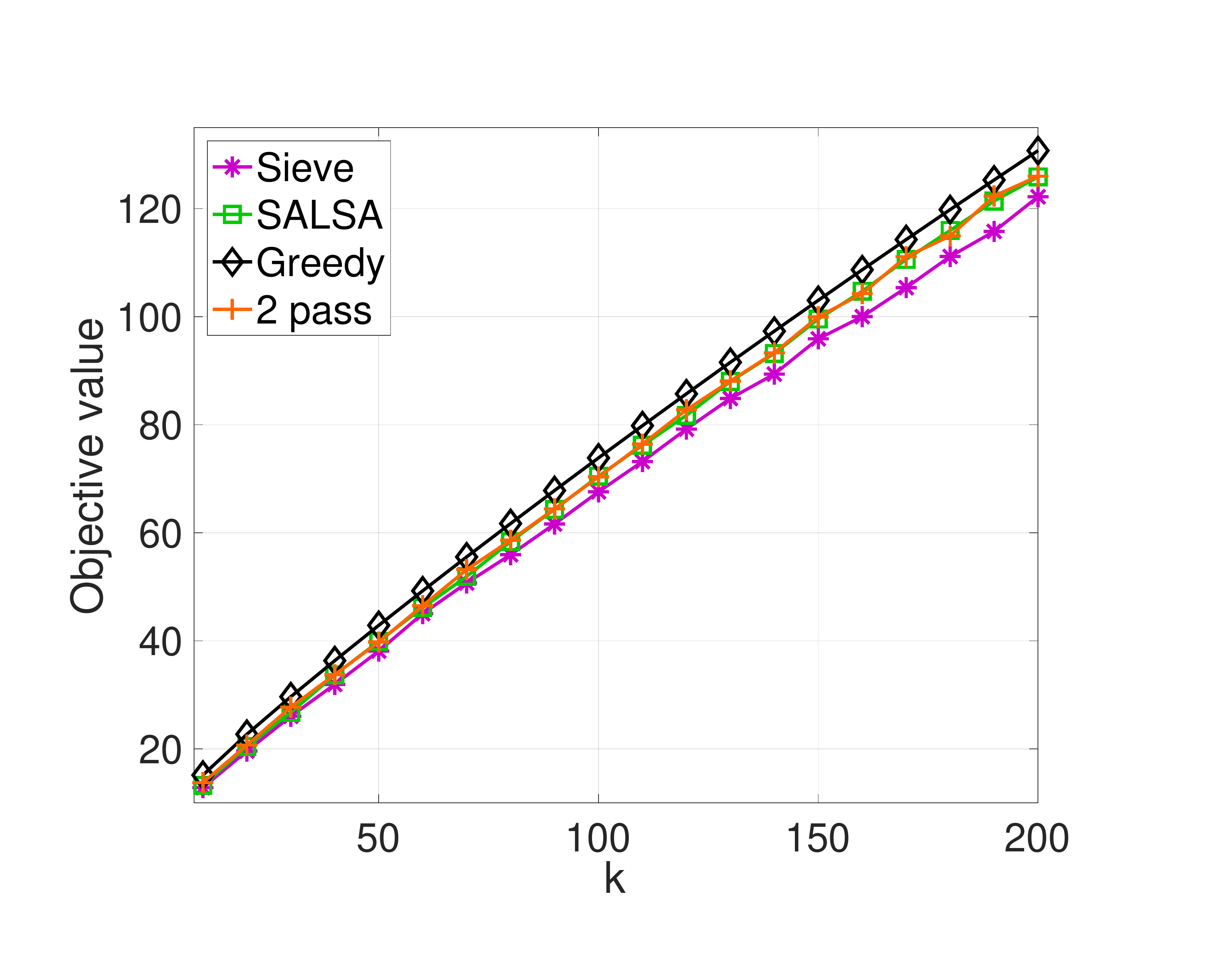}
  \vspace*{\spaceafterincludegraphics}
  \caption*{(i) Movies, $\alpha = 0.85$}
\endminipage

\caption{Numerical comparisons of our two algorithms (\ouralgo and \twopass) and baselines (\greedy and \sieve). In plot (b) we could not run \greedy on the underlying dataset due to its prohibitively slow running time on this instance. Each plot demonstrates the performance of the algorithms for varying values of the cardinality $k$. The datasets used for plots (a)-(e) are described in Section~\ref{sec:exp-coverage}, for plots (f) and (g) in Section~\ref{sec:exp-exemplar}, and for plots (h) and (i) in Section~\ref{sec:exp-movies}.} 
\label{figure:plots}
\end{figure*}

\input{510-dominatig-set}

\input{520-exemplar-based}
\input{530-movie-recommendation}

%% file: 510-dominatig-set.tex
\subsection{Maximum coverage in big graphs}
\label{sec:exp-coverage}

Maximum coverage is a classic graph theory problem with many practical applications, including influence maximization in social networks~\cite{DBLP:journals/toc/KempeKT15} and community detection in graphs~\cite{DBLP:conf/valuetools/FortunatoL09}.
The goal in this problem is to find a small subset of vertices of a graph that is connected to a large fraction of the vertices.

Maximum coverage can be cast as maximization of a submodular function subject to a cardinality constraint. More formally, we are given a graph $G=(V,E)$, where $n=|V|$ denotes the number of vertices and $m=|E|$ denotes the number of edges. The goal is to find a set $S \subseteq V$ of size $k$ that maximizes the number of vertices in the neighborhood of $S$.\footnote{This problem has been also referred to as the dominating set problem in the literature.} We consider three graphs for this problem from the SNAP data library~\cite{snapnets}.
\begin{description}
\item[Pokec social network] Pokec is the most popular online social network in Slovakia. This graph has $n=1,632,803$ and $m=30,622,564$.
\item[LiveJournal social network] LiveJournal~\cite{Backstrom:2006:GFL:1150402.1150412} is a free online community that enables members to maintain journals and individual and/or group blogs. This graph has  $n=4,847,571$ and $m=68,993,773$.
\item[Orkut social network] Similar to Pokec, Orkut \cite{DBLP:journals/kais/YangL15} is also an online social network. This graph has $n=3,072,441$ vertices and $m=117,185,083$ edges.
\end{description}

We compare our algorithms, \ouralgo and \twopass, with both baselines on these datasets for different values of $k$ -- from $100$ to $10,000$. The results show that \ouralgo always outperforms \sieve by around $10\%$, and also reduces the gap between \greedy and the best streaming algorithm by a factor of two. Furthermore, the performance of our \twopass algorithm is very close to that of \greedy. The results can be found in \cref{figure:plots}, where (a) and (b) correspond to the Orkut dataset, (c) and (d) correspond to LiveJournal, and (e) to Pokec.

%% file: 520-exemplar-based.tex
\subsection{Exemplar-based clustering}
\label{sec:exp-exemplar}

Imagine that we are given a collection of emails labeled as spam or non-spam and asked to design a spam classifier. In addition, every email is equipped with an $m$-dimensional vector corresponding to the features of that email. One possible approach is to view these $m$-dimensional vectors as points in the Euclidean space, decompose them into $k$ clusters and fix a representative point for each cluster. Then, whenever a new email arrives, it is assigned the same label as the cluster representative closest to it. Let $V$ denote the set of all the labeled emails. To obtain the described set of cluster representatives, we maximize the following submodular function:
\[
	f(S) = L(\{e_0\}) - L(S \cup \{e_0\}),
\]
where $e_0$ is the all-zero vector, and $L(S)$ is defined as follows~\cite{gomes10budgeted}:
\[
	L(S) = \frac{1}{|V|} \sum_{e \in V} \min_{v \in S} d(e, v).
\]
In the definition of the function $L(S)$, $d(x, y) = \|x - y\|^2$ denotes the squared Euclidean distance.\footnote{Notice that we turn a minimization problem over $L(S)$ into a maximization problem over $f(S)$. The approximation guarantee for maximizing $f(S)$ does not transfer to an approximation guarantee for minimizing $L(S)$. Nevertheless, maximizing $f(S)$ gives very good practical performance, and hence we use it in place of $L(S)$.}

Similarly to spam classification, and among many other applications, the exemplar submodular function can also be used for image clustering. In light of these applications, we perform experiments on two datasets:
\begin{description}
	\item[Spambase] This dataset consists of $4,601$ emails, each email described by $57$ attributes~\cite{Lichman:2013}. We do not consider mail-label as one of the attributes.
	
	\item[CIFAR-10] This dataset consists of $50,000$ color images, each of size $32 \times 32$, divided into $10$ classes. Each image is represented as a $3,072$-dimensional vector -- three coordinates corresponding to the red, green and blue channels of each pixel~\cite{krizhevsky2014cifar}.
\end{description}
Before running these experiments, we subtract the mean of the corresponding dataset from each data point.

The results for the Spambase dataset are shown in \cref{figure:plots}(f). We can observe that both of our algorithms attain a significantly higher utility than $\sieve$. Also, at their point of saturation, our algorithms equalize with \greedy. We can also observe that \sieve saturates at a much lower value than our algorithms, which suggests that the strategy we develop filters elements from the stream more carefully than $\sieve$ does.

Our results for the CIFAR-10 dataset, depicted in \cref{figure:plots}(g), show that, before the point of saturation our algorithms select elements of around $5\%$ higher utility than $\sieve$.
After the point of saturation our algorithms achieve the same utility as $\greedy$, while $\sieve$ approaches that value slowly.
Saturation happens around $k = 10$, which is expected since the images in CIFAR-10 are decomposed into $10$ classes.

%% file: 530-movie-recommendation.tex
\subsection{Personalized movie recommendation}
\label{sec:exp-movies}

We use the Movielens 1M dataset~\cite{harper2016movielens}
to build a recommender system for movies.
The dataset contains over a million ratings
for 3,900 movies
by 6,040 users.
For a given user $u$
and a number $k$,
the system should
recommend a collection of $k$ movies
personalized for user $u$.

We use the scoring function proposed by Mitrović et al.~\yrcite{mitrovic2017streaming}.
We first compute low-rank feature vectors $w_u \in \bR^{20}$ for each user $u$
and $v_m \in \bR^{20}$ for each movie $m$.
These are obtained via low-rank matrix completion
\cite{troyanskaya2001missing}
so as to make each inner product $\ab{w_u, v_m}$
approximate the rating of $m$ by $u$,
if known.
Now we define the submodular function
\[ f_{u, \alpha}(S) = 
\alpha \cdot \sum_{m' \in M}
\max_{m \in S} \ab{v_{m'}, v_{m}}
+
(1 - \alpha)
\cdot
\sum_{m \in S} \ab{w_u, v_m}
. \]
The first term is a
facility-location objective
\cite{lindgren2016leveraging}
that measures how well $S$ covers the space $M$ of all movies
(thus promoting diversity).
The second term aggregates the user-dependent scores
of items in $S$.
The parameter $\alpha$ can be adjusted depending on
the user's preferences.

Our experiments consist in
recommending collections of movies
for $\alpha = 0.75$ and values of $k$ up to $60$
(see \cref{figure:plots}(h)),
as well as
for $\alpha = 0.85$ and values of $k$ up to $200$ 
(see \cref{figure:plots}(i)).
We do this for 8 randomly selected users and report the averages.
We find that
the performance of
both \ouralgo and \twopass
falls at around 40\%
of the gap between \sieve and \greedy.
This quantity improves as $k$ increases.


%% file: 600-conclusion.tex
\section{Conclusion}
In this paper, we consider the monotone submodular maximization problem subject to a cardinality constraint. For the case of adversarial-order streams, we show that a $\nicefrac{1}{2}$ approximation guarantee is tight. Motivated by real-world applications, we also study this problem in random-order streams. We show that the previously known techniques are not sufficient to improve upon $\nicefrac{1}{2}$ even in this setting. We design a novel approach that exploits randomness of the stream and achieves a better-than-$\nicefrac{1}{2}$ approximation guarantee. We also present a multi-pass algorithm that approaches $\rb{1-\nicefrac{1}{e}}$-approximation using only a constant number of passes, even in adversarial-order streams. We validate the performance of our algorithm on real-world data. Our evaluations demonstrate that we outperform the state of the art \sieve algorithm by a considerable margin. In fact, our results are closer to \greedy than to \sieve.
Although we make a substantial progress in the context of streaming submodular maximization, there is still a gap between our approximation guarantee and the currently best known lower bound. It would be very interesting to reduce (or close) this gap, and we hope that our techniques will provide insight in this direction.


%% file: acknowledgements.tex
\subsection*{Acknowledgements}
We thank the anonymous reviewers for their valuable feedback.
Ola Svensson and Jakub Tarnawski were supported by ERC Starting Grant 335288-OptApprox.

%% file: high-density-case.tex

\section{Analysis of the Algorithm}\label{mainalgo}

In this section we analyze our algorithms and present the proof of Theorem \ref{main-theorem}. Throughout this section, we assume that the value $\fopt$ of the optimum solution is known. We remove this assumption in Appendix \ref{sec:removing-opt}.

We run three procedures in parallel and return the best of those as our solution. Algorithm \ref{algo1} works well for instances containing a dense set (see Definition~\ref{def:dense}). The other two algorithms (\cref{algo2,algo3}) guarantee that we attain a high-utility solution in the absence of the density assumption. We prove the correctness of the stated algorithms under the assumption that  $k>k_0= 2 \cdot 10^8$, where $k_0$ is a constant. In \cref{smallkcase} we introduce \cref{alg0}, which completes the proof for the case when $k$ is small.

\maintheorem*

\subsection{The dense case}
In this section we analyse the correctness of Algorithm \ref{algo1} under the assumption that  $k>k_0 = 2 \cdot 10^8$. Let us first define a dense set.

\begin{definition} \label{def:dense}
We say that a set $D$ of elements
is \emph{dense}
if it has
$|D| \le \eta k$
and
$f(D) \ge \frac{1 - \gamma}{2} \fopt$,
where we set $\gamma = 10^{-2}$ and $\eta = 5 \cdot 10^{-5}$.
\end{definition}

This section is devoted to the proof of the following theorem.
\begin{theorem} \label{thm:dense}
There is an algorithm giving a $0.50025$-approximation with probability at least $0.01$
for instances containing a dense set $D$
and having $k \ge k_0$.
\end{theorem}

Consider Algorithm \ref{algo1} with the following values of the parameters: $C_1 = 100$, $C_2 = 10$, $\beta = 0.9$.
In the first $90\%$ of the stream we collect elements of marginal value larger than $\frac{100}{k} \fopt$, and in the remaining $10\%$ of the stream we collect elements of marginal value larger than $\frac{1}{10k}\fopt$. Intuitively, we expect to see $90\%$ of the elements of $D$ in the first $90\%$ of the stream, therefore, with a high enough threshold, we can pick almost all of those elements.
In the remaining $10\%$ of the stream we enhance our solution using a smaller threshold.


Let $S$ be the set of elements collected by this algorithm,
$D$ be a dense set,
and $\cO$ be an optimum solution. For any set of elements $T$, 
define $T_L$ and $T_R$
to be the elements of $T$ in the left part (90\%) of the stream
and in the right part (10\%), respectively.

\begin{fact} \label{fact:positive_prob}
With probability at least $0.01$
we have all of the following:
\begin{align}
|D_L| &\le 0.95 \eta k \,, \label{eq:GLetak} \\
f(D_L) &\ge 0.85 f(D) \,, \label{eq:GL0.85} \\
|\cO_R| &\le 0.15 k \,, \label{eq:OR0.15} \\
f(\cO_R | D) &\ge 0.05 f(\cO | D) \label{eq:OR0.05} \,.
\end{align}
\end{fact}
\begin{proof}
First let us show that
\[ \Pr{\text{not } \eqref{eq:GLetak}} = \Pr{|D_L| > 0.95 \eta k} < 0.001 \,. \]
We apply a standard Chernoff bound on the
indicator variables of the elements of $D$ being in the left part,
which are negatively correlated.
We have $0 \le |D| \le \eta k$,
with $|D| = \eta k$ being the worst case for this bound.
We obtain that \[ \Pr{|D_L| > 0.95 \eta k} < \exp\rb{{-\frac{\rb{0.05/0.9}^2 0.9 \eta k}{3}}}
< 0.001 \,, \]
where we used that $\eta = 5 \cdot 10^{-5}$ and $k \ge k_0 = 2 \cdot 10^8$.

For \eqref{eq:OR0.15},
we use a Chernoff bound again to show that
\[ \Pr{\text{not } \eqref{eq:OR0.15}} = \Pr{|\cO_R| > 0.15 k} < \exp\rb{-\frac{\rb{1/2}^2 0.1 k}{3}} < 0.001 \,. \]

For \eqref{eq:GL0.85},
we will first prove that
$\E{f(D_L)} \ge 0.9 f(D)$.
To that end,
let $D = \{ e_1, ..., e_{|D|} \}$
and define
$Z_i = \one_{e_i \in D_L}$.
Note that
\begin{align*}
\E{f(D_L)}
&= \E{ \sum_{i=1}^{|D|} f(e_i | \{ e_j \in L : 1 \le j \le i-1 \}) Z_i} \\
&\ge \E{ \sum_{i=1}^{|D|} f(e_i | e_1, ..., e_{i-1}) Z_i} \\
&= \sum_{i=1}^{|D|} f(e_i | e_1, ..., e_{i-1}) \E{Z_i} \\
&= 0.9 \sum_{i=1}^{|D|} f(e_i | e_1, ..., e_{i-1}) \\
&= 0.9 f(D) \,.
\end{align*}

Now applying Markov's inequality
(and using that $f(D_L) \le f(D)$)
yields
$\Pr{f(D_L) < 0.85 f(D)} \le \frac23$.

For \eqref{eq:OR0.05},
define the submodular function $g(\cdot) = f(\cdot | D)$
for brevity.
We are interested in lower-bounding the quantity
$\Pr{\eqref{eq:OR0.05} | \eqref{eq:GL0.85}} = \Pr{g(\cO_R) \ge 0.05 g(\cO) | f(D_L) \ge 0.85 f(D)}$.
For this, notice that events \eqref{eq:GL0.85} and \eqref{eq:OR0.05}
are almost independent (they would be independent in the limit $n \to \infty$),
and so we can essentially repeat the analysis for~\eqref{eq:GL0.85}.
Formally, let $I$ be a random variable holding all information about the locations of elements of $D$ in the stream.
(The event \eqref{eq:GL0.85}, as well as the variable $|D_L|$, are known given $I$.)
We would like to prove a lower bound on $\E{ g(\cO_R) | I }$.
This is done as above,
with the difference that we set
$\cO = \{e_1, ..., e_k\}$
and
$Z_i = \one_{e_i \in \cO_R}$.
Now, for any $e_i \in \cO \setminus D$
we can bound
\[
\E{Z_i | I}
= \frac{0.1n - |D_L|}{n - |D|}
\ge \frac{0.1n - \eta k}{n}
= 0.1 - \frac{0.9 \eta k}{n}
\ge 0.1 - \frac{0.9 \frac{n}{100}}{n}
> 0.09
\]
using that $k \le n$ and $\eta \le 10^{-2}$.
On the other hand, for any $e_i \in \cO \cap D$ we have $g(e_i | e_1, ..., e_{i-1}) = g(e_i) = 0$.
Thus we get that $\E{g(\cO_R) | I} \ge 0.09 g(\cO)$ and
\[ \E{g(\cO_R) | \eqref{eq:GL0.85}} = \E{ \E{g(\cO_R) | I} | \eqref{eq:GL0.85}} \ge 0.09 g(\cO) \,. \]
Applying Markov's inequality yields that
$\Pr{g(\cO_R) < 0.05 g(\cO) | \eqref{eq:GL0.85}} \le \frac{0.91}{0.95} < 0.96$.
Finally, we get
\begin{align*}
\Pr{\eqref{eq:GLetak}, \eqref{eq:GL0.85}, \eqref{eq:OR0.15}, \eqref{eq:OR0.05}}
&\ge
\Pr{\eqref{eq:GL0.85}, \eqref{eq:OR0.05}} - \Pr{\text{not \eqref{eq:GLetak}}} - \Pr{\text{not \eqref{eq:OR0.15}}} \\
&\ge
\Pr{\eqref{eq:GL0.85}} \Pr{\eqref{eq:OR0.05} | \eqref{eq:GL0.85}} - \Pr{\text{not \eqref{eq:GLetak}}} - \Pr{\text{not \eqref{eq:OR0.15}}} \\
&\ge
\frac13 \cdot 0.04 - 0.001 - 0.001 \\
&>
0.01
\,.
\end{align*}
\end{proof}

\begin{lemma}\label{lem:dense_set}
Assume 
that the events
\eqref{eq:GLetak}, \eqref{eq:GL0.85}, \eqref{eq:OR0.15}, \eqref{eq:OR0.05}
happen.
Then we have
$f(S) \ge 0.50025  \cdot \fopt$.
\end{lemma}
\begin{proof}
We consider two cases: $|S| < k$ and $|S| = k$.

\paragraph{Case $|S|<k$:}
	First, by the design of the algorithm, for each $e \in D$ we have
	\[ f(e|S) \le \frac{100}{k}\fopt \]
	regardless of whether $e$ appears in the left or in the right part of the stream. Since $|D| \le \eta k$, this implies
	\[ f(D|S) \le \sum_{e\in D}f(e|S) \le 100 \eta \fopt \,. \]
	Moreover, we have \eqref{eq:OR0.05}, i.e., that $f(\cO_R|D) \ge 0.05 f(\cO|D)$. Hence
	\begin{align*}
f(\cO_L| \cO_R \cup D) &= f(\cO | D) - f(\cO_R | D)\\
&\le f(\cO | D) - 0.05 f(\cO | D)\\
&= 0.95 f(\cO | D) \,.
	\end{align*}
	Also, by the design of the algorithm, for every $e \in \cO_R$ it holds that
	\[ f(e| S ) \le \frac{1}{10k}\fopt \,. \]
	Hence, by \eqref{eq:OR0.15},
	\[ f(\cO_R | S) \le \sum_{e \in \cO_R} f(e| S ) \le 0.03 \frac{\fopt}{2} \]
	and therefore
	\begin{align*}
f(\cO | D \cup S) &= f(\cO_R | D \cup S ) + f( \cO_L| \cO_R \cup D \cup S ) \\
&\le f(\cO_R | S) + f(\cO_L| \cO_R \cup D)\\
&\le 0.95 f(\cO | D) + 0.03 \frac{\fopt}{2}.
	\end{align*}
	Therefore,
	\[ f(\cO \cup D | S) = f( D | S) + f(\cO | S \cup D ) \le 100 \eta \fopt + 0.95 f(\cO | D) + 0.03 \frac{\fopt}{2} \,. \]
	Hence,
	\begin{align*}
	f(S) &= f(\cO \cup D \cup S) - f(\cO \cup D | S)\\
	&\ge f(\cO \cup D) - \rb{100 \eta \fopt + 0.95 f(\cO | D) + 0.03 \frac{\fopt}{2}}\\
	&= f(D) + \marginal{\cO}{D} - \rb{100 \eta \fopt + 0.95 f(\cO | D) + 0.03 \frac{\fopt}{2}}\\
	&= 0.95f(D) + 0.05(f(D) + \marginal{\cO}{D}) - 100 \eta \fopt - 0.03 \frac{\fopt}{2}\\
	&\ge 0.95 \frac{1 - \gamma}{2} \fopt + 0.05 \fopt - 100 \eta \fopt - 0.03 \frac{\fopt}{2} \\
	&= \rb{\frac12 + 0.01 - 100 \eta - 0.475 \gamma} \fopt \\
	&= 0.50025 \cdot \fopt \,.
	\end{align*}

\paragraph{Case $|S| = k$:}
	Note that if $|S_L| \ge \frac{k}{100}$,
	then $f(S_L) \ge \fopt$ and we are done.
	Thus for each $e \in D_L$ we have
	\[ f(e| S_L) \le \frac{100}{k} \fopt \,. \]
	Hence, by \eqref{eq:GLetak},
	\[ f(D_L|S_L) \le \sum_{e \in D_L} f(e|S_L) \le 95 \eta \fopt \,. \]
	Therefore, using \eqref{eq:GL0.85},
	\begin{align*}
		f(S_L) &= f(D_L \cup S_L) - f(D_L|S_L) \\
		&\ge f(D_L) - f(D_L| S_L ) \\
		&\ge 0.85f(D) - 95 \eta\fopt \\
		&\ge \rb{\frac{0.85(1-\gamma)}{2} - 95 \eta}\fopt \,.
	\end{align*}

	Recall that $|S_L| < \frac{k}{100}$,
	so that $|S_R| \ge 0.99 k$. We can write
	\[ f(S_R|S_L) \ge \frac{1}{10} 0.99 \cdot \fopt \]
	and
	\begin{align*}
		f(S) &= f(S_L) + f(S_R|S_L) \\
		&\ge \rb{\frac{0.85(1-\gamma)}{2} - 95 \eta + \frac{0.99}{10}}\fopt \\
		&= \rb{\frac12 + 0.024 - 0.425 \gamma - 95 \eta} \fopt \\
		&= 0.515 \cdot \fopt \,.
	\end{align*}
\end{proof}

\cref{fact:positive_prob,lem:dense_set} together imply \cref{thm:dense}.

%% file: general-case.tex
\subsection{General case}
In this section we analyze the correctness of \cref{algo2,algo3} under the assumption that $k>k_0 = 2 \cdot 10^8$.


We invoke \cref{algo2} with the threshold value of $\left(\frac{1+10^{-8}}{2}\right)\left(\frac{\fopt}{k}\right)$, and \cref{algo3} with the value $\beta=10^{-3}$, the threshold value of $\left(\frac{1+10^{-8}}{2}\right)\left(\frac{\fopt}{k}\right)$ for the first $\beta$-fraction of the stream, and the threshold value $\left(\frac{1-3\cdot10^{-11}}{2}\right)\left(\frac{\fopt}{k}\right)$ for the remaining fraction.

Let $\epsilon=10^{-8}$ and $\delta=3\cdot10^{-11}$. We partition the stream into two parts: the left part containing the first $\beta$-fraction
of the arriving elements and the right part containing the remaining $(1-\beta)$-fraction. As presented in \cref{sec:overview-alg}, both \cref{algo2,algo3} act in the same way on the left part of the stream (if the arriving element $e$ satisfies $f(e | S) \geq \left( \frac{1+\epsilon}{2} \right)\left( \frac{\fopt}{k} \right)$, they add it to $S$). However, for the right part of the stream, they proceed with two different strategies. \cref{algo3} works well when the elements selected in the left part carry a lot of value
--
more precisely,
when the value of the left part of the solution is at least $\alpha \fopt$,
where we select $\alpha = 3 \cdot 10^{-3}$.
\cref{algo2} works well in the converse case. 

Let $\cO = \{o_1, \ldots, o_k\}$ denote the optimal solution.
Moreover,
for any set of elements $T$, 
define $T_L$ and $T_R$
to be the elements of $T$ in the left and in the right part of the stream,
respectively.

\begin{claim} \label{claim:chernoff}
  We have
  \begin{align}
    \label{eq:optintersectsL} 
    0.9 \beta k  \leq |\cO_L| \leq 1.1 \beta k
  \end{align}
  with probability at least $\ge 0.999$.
\end{claim}
\begin{proof}
We use a standard Chernoff bound for negatively correlated variables
and get that
\[ \Pr{\abs{|\cO_L| - \beta k} > 0.1 \beta k} \le 2 e^{- \frac{0.1^2 \beta k}{3}} < 0.001 \,, \]
where we used that $k \ge k_0 = 2 \cdot 10^8$ and $\beta = 10^{-3}$.
\end{proof}






We now analyze \cref{algo3}. 
 We do not use any randomness here (beyond assuming \eqref{eq:optintersectsL}). So fix a random arrival. Recall that $S_L$  and $S_R$ are the elements selected in the left and the right part of the stream, respectively.
\newcommand{\strategyoneratiogain}{
	\min\sb{\frac{(1-1.1\beta)\delta - 1.1 \beta \eps}{2}, \alpha \cdot \frac{\eps + \delta}{1 + \eps} - \frac{\delta}{2}}
}
\begin{lemma}
  Assume that $f(S_L) \geq \alpha \fopt$
  and that~\eqref{eq:optintersectsL} is satisfied.
  Then \cref{algo3}  
   outputs a set $S$ such that
  $f(S) \ge \rb{0.5 + 9 \cdot 10^{-12}} \fopt$.
  \label{lem:strategy1}
\end{lemma}
\begin{proof}
  We divide the proof into two cases based on the cardinality of the output set $S$.
\begin{description}
  \item[Case $|S| < k$:] In this case we do not need to use the assumption $f(S_L) \geq \alpha \fopt$. We have $f(S \cup \cO) \ge \fopt$. And $f(S) \geq
    f(S \cup \cO) - \sum_{i=1}^k f(o_i | S)$.  Now, by the definition
    of the algorithm, we have
    \begin{align*}
      f(o_i | S) \leq \begin{cases} \left( \frac{1+\epsilon}{2} \right)\left( \frac{\fopt}{k} \right) & \mbox{if $o_i \in \cO_L$} \\[2mm]
        \left( \frac{1-\delta}{2} \right)\left( \frac{\fopt}{k} \right) & \mbox{otherwise (if $o_i\in \cO_R$).}
      \end{cases}
    \end{align*}
    Hence, using that $|\cO_L| \leq  0.11 k$, we have 
    \begin{align*}
      f(S) & \geq f(S \cup \cO) - \sum_{i=1}^k f(o_i | S) \\
       & \geq \fopt - \sum_{i=1}^k f(o_i | S)  \\
       & \geq \fopt - 1.1 \beta k \cdot \left( \frac{1+\epsilon}{2} \right)\left( \frac{\fopt}{k} \right)  - (1 - 1.1 \beta) k \cdot \left( \frac{1-\delta}{2} \right)\left( \frac{\fopt}{k} \right)   \\
       &= \rb{\frac12 - 1.1 \beta \frac{\eps}{2} + (1-1.1\beta) \frac{\delta}{2}} \fopt \\
       &\ge \rb{0.5 + 9 \cdot 10^{-12} \fopt} \,,
    \end{align*}
    where we used that $\beta = 10^{-3}$, $\eps = 10^{-8}$ and $\delta = 3 \cdot 10^{-11}$.

  \item[Case $|S| = k$:] In order to minimize $f(S)$ we select 
    \[ 
      k - \frac{f(S_L)}{\left( \frac{1+\epsilon}{2} \right)\left( \frac{\fopt}{k} \right)} = k \left(1 - \frac{f(S_L)}{\fopt} \cdot \frac{2}{1+\epsilon}\right)
    \]
    elements in the right part of the stream, each of value $\left( \frac{1-\delta}{2} \right)\left( \frac{\fopt}{k} \right)$. This yields the following lower bound on $f(S)$:
    \begin{align*}
      f(S) &\geq f(S_L) + k \left(1 - \frac{f(S_L)}{\fopt} \cdot \frac{2}{1+\epsilon}\right) \cdot \left( \frac{1-\delta}{2} \right)\left( \frac{\fopt}{k} \right) \\
           &= f(S_L)\left( 1 - \frac{1-\delta}{1+\epsilon} \right) +\left( \frac{1-\delta}{2} \right) \fopt \\
           &\ge \rb{\alpha \cdot \frac{\eps+\delta}{1+\eps} + \frac{1}{2} - \frac{\delta}{2}} \fopt \\
           &\ge \rb{0.5 + 10^{-11}} \fopt \,,
    \end{align*}
    where we used that $\alpha = 3 \cdot 10^{-3}$, $\eps = 10^{-8}$ and $\delta = 3 \cdot 10^{-11}$.
\end{description}
\end{proof}



\begin{theorem} \label{thm:str2}
If there is no dense subset (see \cref{def:dense})
and if $\Pr{f(S_L) \le \alpha \fopt} \ge 0.99$,
then \cref{algo2} 
 returns a set $S$ that has
\[ f(S) \ge \frac{1+\eps}{2} \fopt \]
with probability at least $0.49 \eta$.
\end{theorem}

\begin{proof}
Let $X_i$ be the indicator random variable for the event that the $i$-th arriving element of $\cO$ is added to $S$.
The following is our main technical lemma:
\begin{lemma} \label{lem:str2}
We have $\E{\sum_{i=1}^{0.9 \beta k} X_i} \ge 0.98 \eta \cdot 0.9\beta k$.
\end{lemma}
\begin{proof}
Denote by $S_{<i}$ the elements selected by 
\cref{algo2} up to (but excluding) the arrival of the $i$-th element of $\cO$.
We have 
\begin{align*}
  \E{\sum_{i=1}^{0.9 \beta k} X_i} = \sum_{i=1}^{0.9 \beta k} \E{X_i} \geq \sum_{i=1}^{0.9\beta k} \E{X_i| f(S_{<i}) \leq \alpha \fopt \text{ and~\eqref{eq:optintersectsL}}} \cdot \Pr{f(S_{<i}) \leq \alpha \fopt \text{ and~\eqref{eq:optintersectsL}}} \,.
\end{align*}
Note that
\[ \Pr{f(S_{<i}) \leq \alpha \fopt \text{ and~\eqref{eq:optintersectsL}}} \ge \Pr{f(S_L) \leq \alpha \fopt \text{ and~\eqref{eq:optintersectsL}}} \ge 0.99 - 0.001 \ge 0.98 \]
by the assumption of \cref{thm:str2} and by \cref{claim:chernoff} (note that if $\eqref{eq:optintersectsL}$ holds, then $S_{<i} \subseteq S_L$ for $i \le 0.9\beta k$). So
\begin{align*}
  \E{\sum_{i=1}^{0.9\beta k} X_i} \ge 0.98 \sum_{i=1}^{0.9\beta k} \E{X_i| f(S_{<i}) \leq \alpha \fopt \text{ and~\eqref{eq:optintersectsL}}} \,.
\end{align*}

Fix $1 \le i \le 0.9\beta k$, and
let $I$ be the random variable that denotes
the position in the stream of the $i$-th element of $\cO$
as well as
the contents of the stream up to (but excluding) that position.
(Note that $S_{<i}$ is known given $I$.)
Conditioning on $I$,
we have
\[ \E{X_i| f(S_{<i}) \leq \alpha \fopt \text{ and~\eqref{eq:optintersectsL}}} = \E{ \E{X_i| I, \eqref{eq:optintersectsL}} | f(S_{<i}) \leq \alpha \fopt \text{ and~\eqref{eq:optintersectsL}}} \]
and we proceed to bound the inner expectation for any fixed $I$ such that $f(S_{<i}) \le \alpha \fopt$.
We apply total expectation again, this time over $o_i$:
\begin{align*}
\E{X_i| I, \eqref{eq:optintersectsL}} =
\sum_{o \in \cO}
\Pr{X_i = 1 | o_i = o, I, \eqref{eq:optintersectsL}}
\cdot \Pr{o_i = o | I, \eqref{eq:optintersectsL}} \,.
\end{align*}
Here the first factor is not random: $X_i = 1$ iff $f(o_i | S_{<i}) \ge \frac{1+\eps}{2} \cdot \frac{\fopt}{k}$,
and we know both $S_{<i}$ (from $I$) and $o_i = o$.
Denote the set of good elements by $G = \{ o \in \cO : f(o | S_{<i}) \ge \frac{1+\eps}{2} \cdot \frac{\fopt}{k} \}$.
Then the first factor is just $\one_{o \in G}$.

Now consider the second factor.
We claim that the distribution for $o_i$ given $I$ and $\eqref{eq:optintersectsL}$
is uniform on the elements of $\cO$ that have not yet appeared on the stream.
This is because the global, uniformly random choice for the order of all elements in the stream
can be broken up into three independent choices:
the positions of elements of $\cO$,
the relative order of elements of $\cO$,
and the relative order of elements of $V \setminus \cO$.
Conditioning on $\eqref{eq:optintersectsL}$ only affects the first part,
and together with $I$ it reveals no information about the order of the yet-unseen elements of $\cO$.
Thus the second factor, i.e., $\Pr{o_i = o | I, \eqref{eq:optintersectsL}}$,
is equal to $0$ for those elements of $\cO$ that have appeared before the $i$-th,
and $1/(k+1-i)$ for the others.
Thus
\[ \E{X_i| I, \eqref{eq:optintersectsL}} = \frac{|\{ o \in \cO : \text{$o$ has not yet appeared and $o \in G$} \}|}{k+1-i} \,. \]
However, no element $o \in G$ could have appeared yet!
For suppose otherwise: since $o$ has marginal contribution at least $\frac{1+\eps}{2} \cdot \frac{\fopt}{k}$ for $S_{<i}$,
\textit{a fortiori} it had at least that marginal contribution at the time when it appeared,
so $o$ should have been taken
(note that $|S_{<i}| < k$, otherwise we could not have $f(S_{<i}) \le \alpha \fopt$);
but then $o \in S_{<i}$ and thus $f(o | S_{<i}) = 0$, a contradiction.
Finally we get
\begin{align} \label{eq:G}
\E{X_i| I, \eqref{eq:optintersectsL}} = \frac{|G|}{k+1-i} \,.
\end{align}

\begin{claim} \label{claim:bigG}
We have $|G| \ge \eta k$ (recall that $\eta = 5 \cdot 10^{-5}$ -- see \cref{def:dense}).
\end{claim}
\begin{proof}
Suppose otherwise.
Then we have
\[ f(\cO \setminus G | S_{<i}) \le \sum_{o \in \cO \setminus G} f(o | S_{<i}) \le k \cdot \frac{1+\eps}{2} \frac{\fopt}{k} = \frac{1+\eps}{2} \fopt \]
and thus
\begin{align*}
\frac{1+\eps}{2} \fopt + f(G | S_{<i})
&\ge f(\cO \setminus G | S_{<i}) + f(G | S_{<i}) \\
&\ge f(\cO | S_{<i}) \\
&\ge f(\cO) - f(S_{<i}) \\
&\ge \fopt - \alpha \fopt \,,
\end{align*}
yielding that
\[
f(G)
\ge f(G | S_{<i})
\ge \rb{1 - \alpha - \frac{1+\eps}{2}} \fopt
= \frac{1 - 6 \cdot 10^{-3} - 10^{-8}}{2} \fopt
\ge \frac{1 - 0.01}{2} \fopt
= \frac{1-\gamma}{2} \fopt
\]
(recall that $\alpha = 3 \cdot 10^{-3}$, $\eps = 10^{-8}$, and $\gamma = 10^{-2}$ -- see \cref{def:dense}).

Together with $|G| < \eta k$, this yields that $G$ is a dense set, whose existence contradicts the assumption of \cref{thm:str2}.
\end{proof}

By \cref{eq:G,claim:bigG}
we get
\[ \E{X_i| I, \eqref{eq:optintersectsL}} = \frac{|G|}{k+1-i} \ge \frac{\eta k}{k} = \eta \]
and finally
\[ \E{\sum_{i=1}^{0.9\beta k} X_i} \ge 0.98 \eta \cdot 0.9\beta k \,. \]
\end{proof}

Having that
$\E{\sum_{i=1}^{0.9\beta k} X_i} \ge 0.98 \eta \cdot 0.9\beta k$
by \cref{lem:str2},
now we use the following fact:
\begin{fact} \label{fact:prob_bound}
Let $\ell \ge 0$ and $X$ be a random variable with
$0 \le X \le \ell$
and
$\E{X} \ge \zeta \ell$.
Then $\Pr{X \ge \zeta \ell / 2} \ge \zeta / 2$.
\end{fact}
\begin{proof}
One applies Markov's inequality to the random variable $\ell - X$.
\end{proof}
Applying \cref{fact:prob_bound} (to $X = \sum_{i=1}^{0.9\beta k} X_i$, $\ell = 0.9\beta k$, and $\zeta = 0.98 \eta$) we get
\begin{align} \label{eq:fraction_of_opt}
\Pr{|S \cap \cO| \ge 0.49 \eta \cdot 0.9\beta k} \ge \Pr{\sum_{i=1}^{0.9\beta k} X_i \ge 0.49 \eta \cdot 0.9\beta k} \ge 0.49 \eta \,.
\end{align}
Furthermore, we have:
\begin{lemma} \label{lem:fraction_of_opt_is_nice}
If $|S \cap \cO| \ge 2 \eps k$,
then $f(S) \ge \frac{1+\eps}{2} \fopt$.
\end{lemma}
\begin{proof}
If $|S| = k$, then we have $f(S) \ge \frac{1+\eps}{2} \fopt$
since every added element had at least $\frac{1+\eps}{2} \cdot \frac{\fopt}{k}$ contribution.
On the other hand,
if $|S| < k$,
then for every element $e$ we have
$f(e|S) < \frac{1+\eps}{2} \cdot \frac{\fopt}{k}$.
Also, by assumption, $|\cO \setminus S| \le (1 - 2 \eps) k$.
Thus
\begin{align*}
\fopt
&\le f(S \cup \cO) \\
&\le f(S) + \sum_{o \in \cO \setminus S} f(o | S) \\
&\le f(S) + (1 - 2 \eps) \frac{1+\eps}{2} \fopt \\
&= f(S) + \frac{1 - 2 \eps + \eps - 2 \eps^2}{2} \fopt \\
&\le f(S) + \frac{1 - \eps}{2} \fopt \,,
\end{align*}
which yields $f(S) \ge \frac{1+\eps}{2} \fopt$.
\end{proof}
\cref{eq:fraction_of_opt,lem:fraction_of_opt_is_nice} finish the proof of \cref{thm:str2}:
we have
$0.49 \eta \cdot 0.9 \beta k > 2 \cdot 10^{-8} k = 2 \eps k$
(where we used that $\eta = 5 \cdot 10^{-5}$, $\beta = 10^{-3}$, and $\eps = 10^{-8}$),
and so
we get $f(S) \ge \frac{1+\eps}{2} \fopt$
with at least a constant ($0.49 \eta$) probability.
\end{proof}

%

%% file: smallk.tex
\subsection{Small-$k$ case} \label{smallkcase}

In this section we describe an algorithm
that gives a $\rb{\frac12 + \Omega(1)}$-approximation
for bounded $k$, i.e., $k<k_0$.
Recall that $k_0 = 2 \cdot 10^8$.
We will prove:

\begin{theorem} \label{thm:smallk}
There is an algorithm (\cref{alg0}) for
streaming submodular maximization
in the random order case
that, for any $k$,
achieves a $\rb{\frac12 + g(k)}$-approximation
in expectation,
for some function $g(k) > 0$.
\end{theorem}
In particular, for $k \le k_0$ \cref{alg0} yields a $\rb{\frac12 + \Omega(1)}$-approximation in expectation.
The proof relies on two claims: \cref{fact:very_high_probability} and \cref{fact:smallk_ratio}.

\begin{algorithm}
\caption{The small-$k$ case} \label{alg0}
\begin{algorithmic}[1]

	\State $S:=\emptyset$
	\For{the $i$-th element $e_i$ on the stream}
		\If  {$f(e_i|S) \geq \frac{\fopt - f(S)}{k}$ and $|S|<k$}
			\State $S:=S\cup \{e_i\}$
		\EndIf  
    \EndFor\\
    \Return $S$

\end{algorithmic}
\end{algorithm}

\begin{fact} \label{fact:very_high_probability}
With probability at least $\frac{1}{k!}$ we will have $|S| = k$ at the end.
\end{fact}
The intuitive reason for this is that,
whenever the algorithm takes a new element and changes its threshold,
there is some element $o$ of $\opt$ that is above that threshold.
With positive probability, $o$ is the next element of $\opt$ on the stream.
As long as $|S| < k$ before $o$ is seen, $o$ will be taken.
In this way, the algorithm takes all elements of $\opt$ that it sees before it has collected $k$ elements.
Given that there are $k$ elements of $\opt$, the algorithm cannot finish with $|S| < k$.
\begin{proof}
Consider the optimum set $\cO = \{o_1, ..., o_k\}$,
in the order that these elements appear on the stream.
For $i = 0, ..., k$,
let $\cE_i$ be the event that
either $o_1, ..., o_i \in S$,
or
$S$ is already full at the time $o_i$ arrives.
Now it is enough to prove that for $i = 1, ..., k$ we have $\Pr{\cE_i | \cE_{i-1}} \ge \frac{1}{k+1-i}$.
Once we have this, we can write
\[ \Pr{|S| = k \text{ at the end}} \ge \Pr{\cE_k} \ge \prod_{i=1}^k \Pr{\cE_i | \cE_{i-1}} \ge \frac{1}{k!} \,. \]

So fix $i$.
We want to show that $\Pr{\cE_i | \cE_{i-1}} \ge \frac{1}{k+1-i}$.
Assuming $\cE_{i-1}$,
there are two cases:
either $S$ is already full at the time $o_{i-1}$ arrives,
or
we have $o_1, ..., o_{i-1} \in S$.
In the former case, $S$ is of course still full when $o_i$ arrives, and so $\cE_i$ holds.
So assume the latter case.
Let $S_{<i} \ni o_1, ..., o_{i-1}$ be the contents of $S$ at the time just before the arrival of $o_i$.
We have
\[ f(\cO | S_{<i}) \ge \fopt - f(S_{<i}) \,, \]
so there exists $o \in \cO$ with
\[ f(o | S_{<i}) \ge \frac{\fopt - f(S_{<i})}{k} \,, \]
and of course $o \in \{o_i, ..., o_k\}$
since all previous elements of $\cO$ are in $S_{<i}$ and thus have marginal value $0$.
Note that at this time, conditioning on
the entire stream before $o_i$ and on the knowledge that the next element will belong to $\cO$,
the distribution of $o_i$ is uniform on the elements of $\cO$ that have not arrived yet.
Thus we have that $o = o_i$ with probability $\frac{1}{k+1-i}$.
If $o = o_i$, then 
our algorithm will indeed pick $o_i$
(unless $|S_{<i}| = k$, in which case $\cE_i$ also holds).
This shows
that $\Pr{\cE_i | \cE_{i-1}} \ge \frac{1}{k+1-i}$.
\end{proof}

\begin{fact} \label{fact:smallk_ratio}
If $|S| = k$, then $f(S) \ge \rb{1 - \frac1e} \fopt$.
\end{fact}
The proof is similar to the analysis of the (non-streaming) algorithm \greedy.
\begin{proof}
Let $e_1, ..., e_k$ be the elements of $S$, in order of insertion.
We show by induction on $i = 0, 1, ..., k$ that
\[ \fopt - f(e_1, ..., e_i) \le \rb{\frac{k-1}{k}}^i \fopt \,. \]
The base case $i = 0$ is trivial.
Fix $i \ge 1$.
The algorithm guarantees that $f(e_i | e_1, ..., e_{i-1}) \ge \frac{\fopt - f(e_1, ..., e_{i-1})}{k}$.
Thus we have
\begin{align*}
\fopt - f(e_1, ..., e_i)
&= \fopt - f(e_1, ..., e_{i-1}) - f(e_i | e_1, ..., e_{i-1}) \\
&\le \fopt - f(e_1, ..., e_{i-1}) - \frac{\fopt}{k} - \frac{f(e_1, ..., e_{i-1})}{k} \\
&= \frac{k-1}{k} \sb{\fopt - f(e_1, ..., e_{i-1})} \\
&\le \rb{\frac{k-1}{k}}^i \fopt \,.
\end{align*}
For $i = k$, we get
\[ f(S) = f(e_1, ..., e_k) \ge \rb{1 - \rb{\frac{k-1}{k}}^k} \fopt \ge \rb{1 - \frac1e} \fopt \,. \]
\end{proof}

\begin{proof}[Proof of \cref{thm:smallk}]
We run Algorithm \ref{alg0}
and \sieve in parallel,
choosing the better of the two solutions at the end.
We always get at least a $1/2$-approximation,
and (by \cref{fact:very_high_probability} and \cref{fact:smallk_ratio})
with probability at least $\frac{1}{k!}$
we get at least a $(1-1/e)$-approximation.
Thus we get at least a $\rb{\frac12 + \frac{1}{k!} \rb{\frac12 - \frac1e}}$-approximation in expectation.
\end{proof}

%% file: put-together.tex
\subsection{Proof of the main theorem}
\label{sec:putting-together}
Now we are ready to complete the proof of Theorem~\ref{main-theorem}.

\maintheorem*
\begin{proof}

If $k<k_0 = 2 \cdot 10^8$, then, by \cref{thm:smallk}, \cref{alg0} achieves a $\rb{\frac12 + \frac{1}{k!} \rb{\frac12 - \frac1e}}$-approximation in expectation. Otherwise, as explained above, we run three algorithms in parallel with \sieve and output the best solution out of the four. Throughout, we let $S^1$, $S^2$, and $S^3$ be the solutions returned by \cref{algo1}, \cref{algo2}, and \cref{algo3} respectively, and let $S^4$ be the solution returned by \sieve (with the standard threshold $\frac12 \frac{\fopt}{k}$).



It is known that \sieve is a $\frac 12$-approximation~\cite{badanidiyuru2014streaming}:

\begin{fact} \label{fact:sieve}
We always have $f(S^4) \ge \frac12 \fopt$.
\end{fact}


\begin{lemma}
We have $\E{\max(f(S^1), f(S^2), f(S^3), f(S^4))} \ge \rb{\frac12 + 8 \cdot 10^{-14}} \fopt$.
\end{lemma}
\begin{proof}
We will have three cases, depending on the (non-random) properties of the instance.

\textbf{Case 1: there exists a dense subset.}
Then
by \cref{thm:dense}
we get that
$f(S^1) \ge 0.50025 \fopt$
with probability at least $0.01$.
On the other hand, \cref{fact:sieve} guarantees that $f(S^4) \ge \frac12 \fopt$ always holds.
Thus
\[ \E{\max(f(S^1), f(S^4))} \ge \rb{0.5 + 0.01 \cdot 0.00025} \fopt \,. \]

\textbf{Case 2: we have $\Pr{f(S_L) > \alpha \fopt} > 0.01$.}
By \cref{claim:chernoff} we have
\[ \Pr{f(S_L) > \alpha \fopt \text{ and } \eqref{eq:optintersectsL}} > 0.01 - 0.001 = 0.009 \,. \]
And whenever $f(S_L) > \alpha \fopt$ and $\eqref{eq:optintersectsL}$,
\cref{lem:strategy1} yields that
\[ f(S^3) \ge \rb{\frac12 + 9 \cdot 10^{-12}} \fopt \,. \]
On the other hand, \cref{fact:sieve} guarantees that $f(S^4) \ge \frac12 \fopt$ always holds.
Thus
\[ \E{\max(f(S^3), f(S^4))} \ge \rb{\frac12 + 0.009 \cdot 9 \cdot 10^{-12}} \fopt \ge \rb{0.5 + 8 \cdot 10^{-14}} \fopt \,. \]

\textbf{Case 3: there is no dense subset, and $\Pr{f(S_L) \le \alpha \fopt} \ge 0.99$.}
Then \cref{thm:str2} yields that
$f(S^2) \ge \frac{1+\eps}{2} \fopt$
with probability at least $0.49 \eta$.
On the other hand, \cref{fact:sieve} guarantees that $f(S^4) \ge \frac12 \fopt$ always holds.
Thus
\[ \E{\max(f(S^2), f(S^4))} \ge \rb{\frac12 + 0.49 \eta \cdot \frac{\eps}{2}} \fopt \ge \rb{0.5 + 10^{-13}} \fopt \,. \]
\end{proof}

Therefore, for any $k$, our algorithm outputs a solution of value at least $\rb{\frac12 + 8 \cdot 10^{-14}} \fopt$ in expectation. 


\end{proof}

%% file: lower_bound.tex

\section{Impossibility Result for Adversarial-Order Streams} \label{sec:hardness}

In this section we prove \cref{thm:hardness} -- an unconditional lower bound on the memory usage of any single-pass streaming algorithm for submodular maximization that is allowed to query the value of the submodular function on feasible sets (ones of cardinality at most $k$) and having approximation factor $1/2 + \epsilon$ and some constant probability of success. Such bounds are usually proved via reductions from communication problems with certain communication complexity lower bounds. Here we reduce the \emph{INDEX} problem to our problem. In what follows, we first define the INDEX problem and then we state a known communication complexity lower bound for this problem. We then present a reduction from INDEX  to streaming submodular maximization.

\textbf{INDEX problem:} We consider a communication game consisting of channel coding, where
\begin{itemize}
\item Alice gets $x \in \{ 0,1 \}^m$ for some integer $m$.
\item Bob gets an integer $ i \in [m]$.
\item The goal is to compute the function $f(x,i) = x_i$.
\end{itemize}
We let $R_{2/3}^{pub}(\text{INDEX})$ denote the minimum number of bits required to be sent from Alice to Bob in order to solve INDEX problem with success probability at least $2/3$. The assumption is that Alice and Bob both have access to public random bits. Notice that the communication in this setting, is \emph{only} from Alice to Bob. We know that this has an $\Omega(m)$ lower bound in the one-way communication model (e.g., see \cite{bar2002information}, \cite{jayram2008one})

\begin{theorem}\emph{(Indexing lower bound)}
For any integer $m$, 
\begin{align*}
R_{2/3}^{pub}(\text{\emph{INDEX}}) \ge \frac{1}{100} m
\end{align*}
\end{theorem}
A more general result, involving the k-party generalized addressing function, appears in \cite{bar2002information}. This theorem shows that in order to solve the Indexing problem with constant success probability, $\Omega(m)$ bits of communication is required.

\textbf{Reduction to submodular maximization:}
We present a reduction from the INDEX problem to streaming submodular maximization problem. In this reduction part of the stream is constructed based on the $x$ vector that Alice has and part of the stream is constructed based on the index $i$ that Bob holds. If there exists a streaming algorithm with small memory, Alice can first feed the algorithm with her part of the stream and then send to Bob the state of the memory. Then Bob can continue the algorithm with the memory state he received from Alice and feed the algorithm with his part of the stream and obtain the solution. Then based on the solution that the algorithm gives, Bob outputs $f(x,i)$. So any lower bound that holds on the communication complexity of indexing problem should also hold on the memory usage of streaming submodular mazimization problem.

Formally, we prove the following theorem. 
\begin{theorem}
For any integer $k > 2 $ and any $\delta>0$, there exist a family of instances of submodular maximization problem such that any algorithm which is allowed to query the value of the submodular function on feasible sets, cardinality at most $k$, with approximation guarantee better than $k/(2k-1)$ and success probability $\delta$, needs at least $\Omega(\delta\frac{n}{k})$ bits of  memory.
\end{theorem}

\begin{proof}
We show a reduction from any instance of the INDEX problem to an instance of streaming submodular maximization problem. Let $\mathcal{U}$ be a universe of size $|\mathcal{U}| = (2k+1) m$. We assign $k$ elements to each of the $x_j$'s that Alice has and one element to the index $i$ that Bob has.

\begin{align*}
	\mathcal{U} = \{ \bar{u}_j^l ; j \in [m] , l \in [k] \} \cup \{ u_j^l ; j \in [m] , l \in [k] \} \cup\{ w_i ; i \in [m] \}
\end{align*} 
Let us now construct the stream:
\begin{enumerate}
	\item For every $j \in [m]$ if $x_j =1$ then Alice inserts ${u}_j^l$ for all $l \in [k]$ into the stream. Otherwise, Alice inserts $\bar{u}_j^l$ into the stream.
	\item Afterwards Bob adds  $w_i$ to the stream. Recall that $i$ is the index that is given to bob in the Indexing problem. 
	
\end{enumerate}

Therefore, the length of the stream is $n = km+1$ where $km$ elements of it are at Alice's side and one element is at Bob's side. It remains to define the submodular function $f$. For simplicity, we define $f$ only for the elements that are on the stream. First note that by design, only one of the $w_i$'s can be present in the stream which is the one that correspond to the index $i$ that Bob holds. Now, let ${V}_i = \{u_j^l, \bar{u}_j^l ; j \in [m]\backslash \{i\}, l \in [k]\} \cup\{ \bar{u}_i^l ; l \in [k] \}$ and ${U}_i = \{{u}_i^l ; l \in [k]\}$. For any $\mathcal{S}$, $f(\mathcal{S})$ is defined as follows:

\begin{align}\label{func-submodul}
	f(\mathcal{S}) = |U_i \cap \mathcal{S}|+
	\begin{cases}
		\min (k, |V_i\cap \mathcal{S}| )  & w_i \notin \mathcal{S} \\ 
		k & \text{otherwise}
	\end{cases}
\end{align}

\begin{observation}\label{Obs-sumodularity}
	The function $f(\cdot)$ as defined in \eqref{func-submodul} is monotone and submodular.
\end{observation}

Now note that because in Alice's side of the stream  $w_i$ is not present and by the assumption that the algorithm is only allowed to query the function value on feasible sets, from Alices point of view $f$ collapses to the following function:
$$f(\mathcal{S}) = |\mathcal{S}|$$
for every set $\mathcal{S}$ a subset of the stream in Alice's side such that $|\mathcal{S}| \le k$. Therefore, this function reveals no information about the index $i$ to Alice.
Let ANS be the solution that the algorithm returns. Then Bob outputs $x_i = 1$ if ANS$ > k$ and $x_i=0$ otherwise. 

Let us now compute the value of the optimum solution to the submodular instance that we constructed (denoted by $\fopt$), depending on the answer to the given instance of the INDEX problem.

\begin{itemize}
	\item{if $x_i = 0$:} Then for any subset $\mathcal{S}$ of the stream we have $\mathcal{S} \cap U_i = \emptyset$ hence by definition $f(\mathcal{S}) \le k$, and in fact $\fopt = k$.
	\item{if $x_i = 1$:} Then $f(\mathcal{S}) = 2k-1$, for $\mathcal{S}=\{w_i\} \cup \{u_i^l| l \in[k-1] \}$ so $\fopt \ge 2k-1$ and in fact, $\fopt = 2k-1$.
\end{itemize}
Therefore, any algorithm for submodular maximization problem that has an approximation guarantee better than $k/(2k-1)$ and works with any constant probability $\delta > 0$, should also use memory at least $\frac{\delta}{10} R_{2/3}^{pub}(${INDEX}$)$. The reason is that we can run $\frac{10}{\delta}$ instances of submodular maximization independently and then take the max at Bob's side. Because each of them has approximation guarantee $k/(2k-1)$ with probability $\delta$ independently, their maximum will have approximation guarantee $k/(2k-1)$ with probability at least $1 - (1-\delta)^\frac{10}{\delta} \ge 1 - e^{-10} \ge 2/3$. Therefore the streaming submodular maximization has to use $\frac{\delta}{10} R_{2/3}^{pub}(${INDEX}$) = \Omega(\delta m) = \Omega( \delta \frac{n}{k})$ space.
\end{proof}
Our reduction shows that even estimating the value of $\fopt$ to within a factor better than $k/(2k-1)$ with any constant probability requires memory $\Omega(\frac{n}{k})$.

%% file: sieve_lower_bound.tex
\section{Hard Example for \sieve with Random Arrival Order}

In this section we show that there exists a randomly ordered stream on which $\Sieve$ outputs a set $S$ of expected value at most $(1/2 + o(1)) \fopt$. We start by showing this claim for an algorithm $\cA$ similar to $\Sieve$. Then, in Theorem~\ref{theorem:sieve-bound}, we show that there is a collection of elements that, when presented as a randomly ordered stream, makes $\Sieve$ and $\cA$ behave identically with probability at least $1 - \delta$, for any fixed $\delta > 0$.

Let $\cA$ be an algorithm for submodular maximization in the streaming setting that takes a set of thresholds $\cT$ as an auxiliary parameter. The algorithm $\cA$ instantiates the following greedy procedure:
\begin{itemize}
  \item For each threshold $\tau \in \cT$ \textbf{in parallel}: Let $S_{\tau} = \emptyset$. Then, while $|S_{\tau}| < k$, do the following for each arriving element $e$:
    \begin{itemize}
      \item If the arriving element $e$ satisfies $f(e | S_{\tau}) \geq \tau$, add $e$ to $S_{\tau}$.
    \end{itemize}
	\item Output $\arg \max_{S_{\tau} : \tau \in \cT} f\rb{S_{\tau}}$.
\end{itemize}
We will show that there exists an optimal solution $\opt$ and a collection $\cM$ of elements with the following property. If the elements of $\cM$ are presented in a random order to $\cA$, then for every $\tau \in \cT$ we have that $f(S_{\tau}) \le (1/2 + o(1)) \fopt$ with high probability. In the rest of the section we exhibit one such collection $\cM$.

\begin{claim}\label{claim:constructing-X-and-stream}
	Let $\opt = \{e_1, \ldots, e_k\}$ and $f(e_i) = \fopt / k$, for every $i$. Let $\cT$ be the set of thresholds used by the algorithm $\cA$. Then, there exists a stream of length $O\rb{k (k^2 |\cT| / \delta)^{|\cT|}}$ on which, when the stream is presented in a random order, the algorithm $\cA$ outputs set $S := \arg \max_{S_{\tau} : \tau \in \cT} f(S_{\tau})$ such that $f(S) \le (1/2 + o(1)) \fopt$ with probability at least $1 - \delta$, for any fixed $\delta > 0$. Furthermore, for every $x \in S_{\tau}$ and every $Y \subseteq S_{\tau} \setminus \{x\}$ it holds that $f(x | Y) = \tau$.
\end{claim}
\begin{proof}
We split the proof into two parts. First, for every $\tau$, we exhibit set $X_{\tau}$ that, as we will see later, constitutes set $S_{\tau}$. In the second part we compose sets $X_{\tau}$ to obtain a random stream having the desired properties.

\paragraph{First part: exhibiting $X_{\tau}$.}

	We consider three cases with respect to the value $\tau$, and for each of them give a construction of set $X_{\tau}$. For this part of the proof, we assume that the stream consists only of set $X_{\tau}$ and $\opt$ presented in that order.
	
	\begin{itemize}
		\item Case $\tau \le \fopt / (2 k)$.
		Let $X_{\tau} = \{x_1, \ldots, x_k\}$ such that $f(x_i) = \tau$, for every $i$, and $f(X_{\tau}) = k \tau$. Clearly, $\cA$ will collect all the $k$ elements of $X_{\tau}$, and hence will not select any element of $\opt$, i.e. $S_{\tau} = X_{\tau}$. Note that $f(X_{\tau}) = k \tau \le \fopt / 2$.
	
	\item Case $\fopt / (2 k) < \tau \le \fopt / k$.
		Let $X_{\tau} = \{x_1, \ldots, x_t\}$, where $t = \fopt / (2 \tau)$ (for the sake of clarity, we assume that $t$ is an integer and remove this assumption at the end of the proof). In addition, we define $X_{\tau}$ so that: $f(X_{\tau}) = t \tau = \fopt / 2$; $f(x_i) = f(X_{\tau}) / t = \tau$; and, $f(e_i| X_{\tau}) = (\fopt - f(X_{\tau})) / k < \tau$. It is easy to see that such set $X_{\tau}$ exists. Hence, in this case, we have $S_{\tau} = X_{\tau}$ and therefore $f(S_{\tau}) = \fopt / 2$.
		
	\item Case $\fopt / k < \tau$.
		In this case, we simply set $X_{\tau} = \emptyset$. Then, as $f(e_i) < \tau$ for every $i \in \opt$, we have $S_{\tau} = S_{\tau} \cap \opt = \emptyset$.
		
	\end{itemize}
	
	\textbf{Additional property of $X_{\tau}$ sets.} Let $\fopt / (2 k) < \tau \le \fopt/k$. From the definition we have $f(X_{\tau}) = \fopt / 2$. For each such $\tau$ design $X_{\tau}$ so that it ``covers'' the same half of $\opt$. For each $\tau' \le \fopt / (2 k)$ design $X_{\tau'}$ so that it covers a subarea of $X_{\tau}$. Then, for $\tau_1 < \tau_2$ and $\fopt / (2 k) < \tau_2 \le \fopt/k$ we have
			\begin{equation}\label{eq:property-X-four}
				f(x | X_{\tau_2}) = 0, \forall x \in X_{\tau_1}.
			\end{equation}
			
	
\paragraph{Second part: composing a random stream.}
	First, observe that for $\tau > \fopt / k$ the algorithm $\cA$ will not collect any element from $\opt$ and also $X_{\tau} = \emptyset$. Therefore, such threshold $\tau$ does not affect the outcome of $\cA$, and for the rest of the proof and w.l.o.g. we assume that for every $\tau \in \cT$ it holds $\tau \le \fopt / k$. Then, from our construction, we have $k/2 \le |X_{\tau}| \le k$.
	
	Let $\fopt / k \ge \tau_1 > \tau_2 > \ldots > \tau_{|\cT|}$ be the thresholds of $\cT$. Let $\cM$ be a multiset of elements that consists of the following:
	\begin{itemize}
		\item The multiset $\cM$ contains $\opt$.
		\item For each $i$, $\cM$ contains $(k^2 |\cT| / \delta)^i$ copies of $X_{\tau_i}$.
	\end{itemize}
	Let $\cM_R$ be a random stream consisting of the elements of $\cM$. For the sake of brevity, define $X_{\tau_0} := \opt$. Then, we have the following. For every $i \ge 1$ and any element $x \in X_{\tau_i}$ it holds
	\begin{eqnarray*}
		& & \Pr{\text{any element of the copies of } X_{\tau_{i - 1}} \text{ appears before all the copies of } x \text{ in } \cM_R} \\
		& \le & \frac{k (k^2 |\cT| / \delta)^{i - 1}}{(k^2 |\cT| / \delta)^i} \\
		& \le & \frac{\delta}{k |\cT|}.
	\end{eqnarray*}
	Therefore, by union-bound, one copy of $X_{\tau_i}$ appears before any of the elements of $X_{\tau_{i - 1}}$ (taking into account of its copies) in $\cM_R$ with probability at least $1 - \delta / |\cT|$. Furthermore, this claim holds for all the $i$ simultaneously with probability at least $1 - \delta$. Let $\cE$ be event that one copy of $X_{\tau_i}$ appears before any of the elements of $X_{\tau_j}$, for all $j < i$. Our discussion implies $\Pr{\cE} \ge 1 - \delta$. In the rest of the proof, assume that $\cE$ was realized.
	
	Now, consider a threshold $\tau_i \in \cT$. First, no element from $X_{\tau_j}$, for $j > i$, will be chosen to $S_{\tau_j}$ by $\cA$ as $\tau_j > \tau_i$. Next, we distinguish two cases.
	\begin{itemize}
		\item Case $\tau_i \le \fopt/(2k)$.
		In this case, we have $|X_{\tau_i}| = k$ by the design of $X_{\tau_i}$. Hence, assuming $\cE$, we have that $S_{\tau_i}$ will contain $k$ elements before any element of $X_{\tau_j}$, for any $j < i$, is seen in the stream. Therefore, $S_{\tau_i} = X_{\tau_i}$, and $f(S_{\tau_i}) \le \fopt/2$.
		
		\item Case $\fopt/(2k) < \tau_i \le \fopt/ k$.
		In this case and assuming $\cE$, $\cA$ will select $X_{\tau_i}$ to $S_{\tau_i}$ before any element of $X_{\tau_j}$, for any $j < i$, is seen in the stream. Now, by the properties of sets $X_{\tau}$, including property~\eqref{eq:property-X-four}, for any $x \in X_{\tau_j}$ and any $j < i$, we have
		\begin{equation}\label{eq:X-overlaps}
			f(x|X_{\tau_i}) = 0,
		\end{equation}
		and hence no such $x$ will be added to $S_{\tau_i}$. Furthermore, the algorithm $\cA$ executed on $X_{\tau_i}$ and $\opt$ (in that order) outputs set $S_{\tau_i}$ such that $f(S_{\tau_i}) \le \fopt / 2$, as desired.
	\end{itemize}

	\textbf{Removing the assumption that $\fopt / (2 \tau)$ is integral.} Recall that this assumption was made in the case $\fopt / (2 k) < \tau \le \fopt / k$. We start by redefining $t$ as $t = \ceil{\fopt / (2 \tau)}$. First, notice that in this case $f(X_{\tau}) \le (1/2 + o(1)) \fopt$ and, furthermore, all the other aforementioned properties hold except property~\eqref{eq:property-X-four}. The only place where we need property~\eqref{eq:property-X-four} is in our analysis of the second part to derive equation~\eqref{eq:X-overlaps}, i.e. to show that once $X_{\tau_i}$ is collected then no element from $X_{\tau_j}$, for any $j < i$, will be added to the set $S_{\tau_i}$. But, to achieve that, instead of equation~\ref{eq:X-overlaps} the following weaker property suffices
	\begin{equation}\label{eq:X-overlaps-weaker}
		f(x|X_{\tau_i}) \le \fopt / (2k) < \tau_i.
	\end{equation}
	Now we exhibit a collection of sets $X_{\tau}$ such that, even in the case $\fopt / (2 \tau)$ is not integral, the collection have all the desired properties (with property~\eqref{eq:X-overlaps-weaker} replacing property~\eqref{eq:X-overlaps}.
	
	Let $f$ be a cover function in the $2$-dimensional space. Assume that $\opt$ is a rectangle. Divide that rectangle into $2 \cdot k! \cdot |\cT|$ small rectangles all of the same area. Let $A$ be the set of a half of those rectangles. Observe that $f(A) = \fopt / 2$. Next, we define every $X_{\tau}$ so that $f(A | X_{\tau}) = 0$ as follows. All the rectangles of $A$ are (arbitrarily) covered by the elements of $X_{\tau}$ so that every element of $X_{\tau}$ covers $|A| / |X_{\tau}|$ rectangles of $A$, and no two elements of $X_{\tau}$ overlap. Observe that $|A|$ is divisible by any positive integer being at most $k$, and hence is divisible by $|X_{\tau}|$. The remaining value $\tau - \fopt / (2 |X_{\tau}|)$ of every element of $X_{\tau}$ that is not contained within $A$ is arbitrarily allocated in the part of $\opt$ outside of $A$, under the condition that no two elements of $X_{\tau}$ overlap. But now, for any $X_{\tau_i}$ and $X_{\tau_j}$ such that $\fopt/(2k) < \tau_i, \tau_j \le \fopt / k$, and for any $x \in X_{\tau_j}$ we have
	\[
		f(x|X_{\tau_i}) \le \tau - \fopt / (2 |X_{\tau_j}|) \le \fopt / k - \fopt (2 k) \le \fopt / (2 k),
	\]
	as desired.
		
	This concludes the proof.
\end{proof}

Now we use Claim~\ref{claim:constructing-X-and-stream} to conclude that $\Sieve$ outputs a set of value at most $(1/2 + o(1)) \fopt$.
\begin{restatable}{theorem}{theoremsievebound}\label{theorem:sieve-bound}
	There is a stream on which, even when the stream is presented in a random order, $\Sieve$ outputs set $S$ such that with probability at least $1 - \delta$ it holds $f(S) \le (1/2 + o(1)) \fopt$, for any fixed $\delta > 0$.
\end{restatable}
\begin{proof}
	Algorithm $\Sieve$ considers a list $\cG$ of guesses of the value of an optimal solution. We point out that $|\cG|$ does not depend on the length of the stream. For each of the guesses $v \in \cG$, the algorithm maintains set $S_v$ that adds an element $e$ to the set if $f(e | S_v) \ge (v/2 - f(S_v)) / (k - |S_v|)$. Let $\tau := v / (2 k)$. Then, as long as every added $e$ has marginal gain exactly $\tau$, we have 
	\[
		\frac{v/2 - |S_v| \tau}{k - |S_v|} = \frac{v/2 - |S_v| v / (2 k)}{k - |S_v|} = \tau.
	\]
	In other words, if every element added to $S_v$ has marginal gain equal to $\tau$, then the threshold $\Sieve$ considers to add a new element to $S_v$ remains the same. But this is exactly how the algorithm $\cA$ will behave, assuming that such stream is presented to $\Sieve$. By Claim~\ref{claim:constructing-X-and-stream}, there exists a stream which when presented randomly has this desired property with probability at least $1 - \delta$. This now shows that there is a stream on which, when given randomly, the algorithms $\Sieve$ and $\cA$ behave exactly the same with probability at least $1 - \delta$, and hence $\Sieve$ outputs set $S$ such that $f(S) \le (1/2 + o(1)) \fopt$.
\end{proof}

%% file: p-pass.tex
\section{\ppass Algorithm}
\label{sec:multi-pass}

In this section, we present a multi-pass algorithm for the \SEXYP problem. We assume that the value $\fopt$ of the optimum solution $\opt$ is known in advance. We remove this assumption in Appendix \ref{sec:removing-opt}. Our algorithm achieves $1-1/e-\eps$ approximation for arbitrary $\eps$ using $O(k)$ memory and $O(\frac{1}{\epsilon})$ passes over the data stream. In \cite{mcgregor2016better} the problem of maximum $k$-set coverage (a special case of \SEXYP) was studied. They give a $\rb{1-1/e-\eps}$-approximation algorithm with $O\rb{\frac{k}{\epsilon^2}}$ space and $O\rb{\frac{1}{\epsilon}}$ passes. 

Our \ppass algorithm (Algorithm~\ref{ppassalgori}) works as follows: We start with $S =\emptyset$, we pick an element in the $i$-th pass over the stream if $|S| < k$ and $f(e|S)$ is at least $T_i\cdot\frac{OPT}{k}$, where $T_i=(\frac{p}{p+1})^i$.   
\begin{algorithm}[h!] 
\caption{\label{ppassalgori} \ppass Algorithm}
\begin{algorithmic}[1]
	\State $S:=\emptyset$
	\For{$i=1$ to $p$}
		\For{the $j$-th element $e_j$ on the stream}
			\If  {$f(e_j|S) \geq {(\frac{p}{p+1})}^i\cdot \frac{\text{OPT}}{k}$ and $|S|<k$}
				\State $S:=S\cup \{e_j\}$
			\EndIf  
    	\EndFor
	\EndFor
    \Return $S$
\end{algorithmic}
\end{algorithm}

Let $S_i$ be the partial solution obtained  after the $i$-th  pass over the stream, for $1 \le i \le p$. Let $k_i=|S_i \setminus S_{i-1}|$ denote the number of elements picked by the algorithm in the $i$-th pass. Let us begin my showing some properties of the $S_i$ sets. We first show that if for some $1\leq i\leq n$, $S_i$ is not full, then $f(S_i)$ is quite big. Formally:
\begin{lemma} \label{not full}
For any $1\leq i \leq p$,  if $|S_i|<k$, then $f(S_i) \geq \fopt(1-T_i)$.
\end{lemma}

\begin{proof}
Since $|S_i|<k$, thus for any element $ o \in \cO \setminus S_i$,  we have $f(o|S_i) \leq T_i \cdot\frac{\fopt}{k}$, and for any element $o \in \cO \cap S_i$, $f(o|S_i)=0$. Therefore, \[f(\cO|S_i)=\sum_{o \in \fopt} f(o|S_i) \leq k\cdot \frac{T_i\cdot \fopt}{k}=T_i \cdot \fopt.\] Thus,
\[\fopt  \leq f(S_i\cup \cO) = f(S_i) + f(\cO|S_i) \Rightarrow\]  
\[
f(S_i) \geq \fopt\cdot(1-T_i).
\]
\end{proof}
This lemma shows that if $|S|\leq k$ at the end of our algorithm, then we get the desired approximation guarantee. Another important ingredient that we need in order to analyze our algorithm is understanding the case that we pick the expected (or slightly more) number of elements in each round of our algorithm. More precisely, we show that:

\begin{lemma} \label{full early}
For any $1\leq i \leq p-1$, if $\sum_{j=1}^i k_j \geq k\cdot \rb{\frac{i}{p}+\alpha}$, then $f(S_i)\geq \fopt\cdot \rb{1-{(\frac{p}{p+1})}^{i}+\alpha {(\frac{p}{p+1})}^{i+1}}$.
\end{lemma}

\begin{proof}
We prove the lemma by induction.

\textbf{Base case}: $i=1$. If $k_1 \geq k\cdot \rb{\frac{1}{p}+\alpha}$, then we have \[f(S_1) \geq k_1\cdot \frac{T_1\cdot \fopt}{k} = k\cdot \rb{\frac{1}{p}+\alpha}\cdot \frac{p}{p+1}\cdot \frac{ \fopt}{k} = \fopt \cdot \rb{1-\frac{1}{p+1}+\alpha \frac{p}{p+1}}.\]

\textbf{Induction hypothesis}: If $\sum_{j=1}^i k_j \geq k\cdot \rb{\frac{i}{p}+\alpha}$, then \[f(S_i)\geq \fopt\cdot \rb{1-\rb{\frac{p}{p+1}}^{i}+\alpha \cdot \rb{\frac{p}{p+1}}^{i+1}}.\] 

\textbf{Induction step}: If $\sum_{j=1}^{i+1} k_j \geq k(\frac{i+1}{p}+\alpha)$, then we have \[f(S_{i+1})\geq \fopt\cdot \rb{1-\rb{\frac{p}{p+1}}^{i+1}+\alpha \cdot \rb{\frac{p}{p+1}}^{i+2}}.\]

Let us consider three following cases,
\begin{enumerate}
\item If $k_{i+1}=k\cdot \rb{\frac{1}{p}+\beta}$ for $\beta \geq \alpha$.

Since $|S_i|<k$, by Lemma \ref{not full}, we have $f(S_i) \geq \fopt\cdot(1-{(\frac{p}{p+1})}^i)$. Thus we have
\begin{align*}
f(S_{i+1}) &\geq f(S_i)+k_{i+1} \cdot \frac{T_{i+1}\cdot \fopt}{k} \\
&\geq \fopt\cdot \rb{1-{(\frac{p}{p+1})}^i} +k\cdot\rb{\frac{1}{p}+\beta} \rb{\frac{p}{p+1}}^{i+1}\cdot\frac{\fopt}{k} \\
&= \fopt\cdot \rb{1-\rb{\frac{p}{p+1}}^{i+1}+\beta \rb{\frac{p}{p+1}}^{i+1} }\\
&\geq \fopt\cdot \rb{1-\rb{\frac{p}{p+1}}^{i+1}+\alpha \rb{\frac{p}{p+1}}^{i+1}}.
\end{align*}

\item If $k_{i+1}=k(\frac{1}{p}+\beta)$ for $0 \leq \beta \leq \alpha$.

Thus, we have $\sum_{j=1}^{i} k_j \geq k\cdot \rb{\frac{i+1}{p}+\alpha} - k\cdot\rb{\frac{1}{p}+\beta} \geq k\cdot\rb{\frac{i}{p}+(\alpha-\beta)}$. Therefore, by induction hypothesis, we have \[f(S_i)\geq \fopt\cdot \rb{1-\rb{\frac{p}{p+1}}^{i}+(\alpha-\beta) \rb{\frac{p}{p+1}}^{i+1}}.\]
Thus,
\begin{align*}
f(S_{i+1}) &\geq f(S_i)+k_{i+1} \cdot \frac{T_{i+1}\cdot \fopt}{k} \\
&\geq \fopt\cdot \rb{1-\rb{\frac{p}{p+1}}^{i}+(\alpha-\beta) \rb{\frac{p}{p+1}}^{i+1}} +k\cdot\rb{\frac{1}{p}+\beta} \rb{\frac{p}{p+1}}^{i+1}\cdot\frac{\fopt}{k} \\
&= \fopt\cdot \rb{1-\rb{\frac{p}{p+1}}^{i+1}+\alpha \rb{\frac{p}{p+1}}^{i+1}}.
\end{align*}

\item If $k_{i+1}=k\cdot \rb{\frac{1}{p}-\beta}$ for $0 \leq \beta \leq \frac{1}{p}$.

Thus, we have $\sum_{j=1}^{i} k_j \geq k\cdot \rb{\frac{i+1}{p}+\alpha} - k\cdot\rb{\frac{1}{p}-\beta} \geq k\cdot{\frac{i}{p}+(\alpha+\beta)}$. Therefore, by induction hypothesis, we have \[f(S_i)\geq \fopt\cdot \rb{1-\rb{\frac{p}{p+1}}^{i}+(\alpha+\beta) \rb{\frac{p}{p+1}}^{i+1}}.\]
Thus,
\begin{align*}
f(S_{i+1}) &\geq f(S_i)+k_{i+1} \frac{T_{i+1}\cdot \fopt}{k} \\
&\geq \fopt\cdot \rb{1-\rb{\frac{p}{p+1}}^{i}+(\alpha+\beta) \rb{\frac{p}{p+1}}^{i+1}} +k\cdot(\frac{1}{p}-\beta) \rb{\frac{p}{p+1}}^{i+1}\cdot\frac{\fopt}{k} \\
&= \fopt\cdot \rb{1-\rb{\frac{p}{p+1}}^{i+1}+\alpha \rb{\frac{p}{p+1}}^{i+1}}.
\end{align*}
\end{enumerate}
Therefore in all cases we proved 
 \[f(S_{i+1})\geq \fopt\cdot \rb{1-\rb{\frac{p}{p+1}}^{i+1}+\alpha \rb{\frac{p}{p+1}}^{i+2}}.\]

\end{proof}

Now we are ready to prove the main result of this section:
\begin{theorem}
$f(S_p) \geq \fopt\cdot\rb{1-\rb{\frac{p}{p+1}}^p}$.
\end{theorem}
\begin{proof}

Let us now consider the following cases:

\begin{enumerate}
\item If $|S|=|S_p|<k$, then by Lemma \ref{not full}, we have,
\[f(S_i) \geq \fopt\cdot (1-T_p) = \fopt\cdot\rb{1-\rb{\frac{p}{p+1}}^p}.\]
\item If $|S|=k$ and $k_p \geq \frac{k}{p}$.

Then $|S_{p-1}|<k$, thus,  using Lemma \ref{not full}, we get,
\begin{align*}
f(S_p) &\geq f(S_{p-1})+k_p \cdot \frac{T_p\cdot \fopt}{k} \\
&\geq \fopt\cdot \rb{1-\rb{\frac{p}{p+1}}^{p-1}} +\frac{k}{p}\cdot \rb{\frac{p}{p+1}}^{p}\cdot\frac{\fopt}{k} \\
&\geq \fopt\cdot\rb{1-\rb{\frac{p}{p+1}}^p}.
\end{align*}

\item If $|S|=k$ and $k_p < \frac{k}{p}$.

Let $k_p=k\cdot (\frac{1}{p}-\alpha)$, for $0\leq \alpha < \frac{1}{p}$. Thus, we have \[\sum_{i=1}^{p-1} k_i \geq k - k\cdot \rb{\frac{1}{p}-\alpha} \geq k\cdot \rb{\frac{p-1}{p}+\alpha}.\]
Therefore using Lemma \ref{full early}, we get,
\begin{align*}
f(S_p) &\geq f(S_{p-1})+k_p \cdot \frac{T_p\cdot \fopt}{k} \\
&\geq \fopt
\cdot \rb{1-\rb{\frac{p}{p+1}}^{p-1}+\alpha \rb{\frac{p}{p+1}}^{p}} + k\cdot\rb{\frac{1}{p}-\alpha} \cdot \rb{\frac{p}{p+1}}^{p}\cdot\frac{\fopt}{k} \\
&\geq \fopt\cdot\rb{1-\rb{\frac{p}{p+1}}^{p}}.
\end{align*}
\end{enumerate}
Therefore in all cases we proved, 
\[ f(S_p) \geq \fopt\cdot\rb{1-\rb{\frac{p}{p+1}}^{p}}.\]
\end{proof}

%% file: remove_opt.tex
\section{Removing Assumption that OPT is Known}\label{sec:removing-opt} 

Any algorithm we introduced so far works under the assumption of knowing the value of the optimum solution in advance. Let $A$ be the representative of one of our algorithms, and $v$ be the estimation of the optimum solution. Observe that all of our algorithms work as follows: It first starts with an empty set $S =\emptyset$ and adds element $e$ to $S$ if $|S| < k$ and $f(e|S) \ge T \frac{v}{k}$ where $T$ is some fixed or adaptive constant depending on the algorithm. Denote by $A(v)$ the output of the algorithm given value $v$ as an estimation of the optimum solution. We proved that for any of our algorithms there is a constant $c$ such that $f(A(v)) \geq c\cdot \fopt$ if $v=\fopt$. It is also easy to see that, if $\alpha \cdot \fopt \leq v \leq \fopt$, for any $0 \leq \alpha \leq 1$, then $f(A(v)) \geq c\cdot v \geq c \cdot \alpha \cdot \fopt$.  

To that end, we use the same approach as explained in \cite{badanidiyuru2014streaming}. Let $O=\{(1+\epsilon)^j|j\in \mathbb{Z}\}$ thus, there exists a value $v \in O$ such that $\frac{\fopt}{1+\epsilon}\leq v \leq \fopt $. Let $S_v=A(v)$ and $S=\text{argmax}_{v \in O} f(S_v)$. Therefore, $f(S) \geq  f(S_v) \geq \frac{c}{1+\epsilon}  \cdot \fopt \ge c\cdot(1-\epsilon) \cdot \fopt$. We wish to run a copy of the algorithm $A$ for any $v \in O$ in parallel, and output the best solution, however $|O|=\infty$.

To deal with this, we keep track of the maximum value element of the stream at any time. Let $m_i=\max_{1\leq j \leq i} f(\{e_j\})$ denote the maximum value element of the stream after observing $e_1, e_2, \ldots, e_i$. Clearly $m_i \leq \fopt \leq k\cdot m_i$. Also notice that the algorithm $A$ given  value $v$ as an estimation for $\fopt$, picks an element $e$ from the stream only if $f(e|S) \geq T\cdot \frac{v}{k}$. 

Therefore it suffices to keep the estimations in $O_i$ within the range $[m_i, \frac{k\cdot m_i}{T}]$. Hence, we define $O_i=\{(1+\epsilon)^j | j\in \mathbb{Z}, m_i\leq (1+\epsilon)^j \leq \frac{k \cdot m_i}{T}\}$. Thus for all $v \in O_i$, we know that any element with the marginal value at least $T\cdot \frac{v}{k}$ appears only after updating $O_i$. Hence for any $v \in O_i\setminus O_{i-1}$, we can start with the empty set $S_v=\emptyset$. Any time $m_i$ gets updated, we delete all $S_v$'s which $v \notin O_i$. We run $|O_i|$ copies of the algorithm $A$ in parallel for any $v \in O_i$.

\begin{algorithm}
\caption{\label{rem-opt} Guessing OPT}
\begin{algorithmic}[1]
	
	\State $m= 0$
	\For{$i=1$ \text{to} $n$}
		\State $m=\max (m, f(\{e_i\})$
        \State $O_i=\{(1+\epsilon)^j | j\in \mathbb{Z}, m\leq (1+\epsilon)^j \leq \frac{k \cdot m}{T}\}$
        \State Delete all $S_v$'s such that $v \notin O_i$
		\State For each $v \in O_i \setminus O_{i-1}$ set $S_v=0$
        \For{$v \in O_i$}
         	\State $S_v=A(v)$
        \EndFor
    \EndFor \\
    \Return $\arg \max_{v \in O_n} f(S_v)$
\end{algorithmic}
\end{algorithm}

Therefore memory, and the update time of the new algorithm increases by the factor $|O_i|=\log_{1+\epsilon} \frac{k}{T} = O\rb{\frac{\log \frac{k}{T}}{\epsilon}}$, and it outputs  $c\cdot(1-\epsilon)$-approximate solution.
